\documentclass{article}

\usepackage[final,nonatbib]{arxiv}
\usepackage[10pt]{moresize}

\usepackage[numbers]{natbib}

\usepackage{amsmath,amsfonts,bm}

\def\eqref#1{equation~\ref{#1}}

\def\1{\bm{1}}

\DeclareMathAlphabet{\mathsfit}{\encodingdefault}{\sfdefault}{m}{sl}
\SetMathAlphabet{\mathsfit}{bold}{\encodingdefault}{\sfdefault}{bx}{n}

\newcommand{\R}{\mathbb{R}}

\DeclareMathOperator*{\argmax}{arg\,max}
\DeclareMathOperator*{\argmin}{arg\,min}

\DeclareMathOperator{\Tr}{Tr}

\usepackage{amsmath}
\usepackage{amsthm}
\usepackage{amsfonts}       
\usepackage{nicefrac}       

\usepackage{hyperref}
\usepackage{url}

\usepackage{overpic}

\usepackage{graphicx}
\usepackage{epstopdf}
\usepackage{subcaption}
\usepackage{multirow}
\usepackage{float}
\usepackage{empheq}
\usepackage{booktabs}

\usepackage{upgreek}
\usepackage{enumitem}

\usepackage{lipsum}
\usepackage{amsfonts}
\usepackage{graphicx}
\usepackage{epstopdf}

\usepackage[utf8]{inputenc} 
\usepackage[T1]{fontenc}    
\usepackage{booktabs}       
\usepackage{siunitx}
\usepackage{amsfonts}       
\usepackage{nicefrac}       
\usepackage{microtype}      

\usepackage[]{amsmath}
\usepackage{upgreek}
\usepackage{overpic}
\usepackage[]{multirow}
\usepackage[]{nicefrac}
\usepackage[]{bm}

\usepackage{amssymb}
\usepackage{graphicx}
\usepackage{epstopdf}
\usepackage{subcaption}
\usepackage{multirow}
\usepackage{float}

\usepackage{todonotes}

\usepackage{tikz}
\usepackage{pgfplots}
\usepackage{mathtools}
\usetikzlibrary{matrix,positioning,shapes,shadows,arrows,calc,decorations.pathreplacing}
\usetikzlibrary{arrows.meta}

\tikzstyle{blocky} = [draw, line width=1.2pt,fill=white, rectangle, 
minimum height=2.5em, minimum width=4.5em, rounded corners]    

\tikzstyle{blocky2} = [draw, line width=1.2pt,fill=red!10, rectangle, 
minimum height=2.5em, minimum width=18.5em, rounded corners]

\definecolor{dkgray}{rgb}{99,99,99}
\definecolor{darkred}{RGB}{228,26,28}
\definecolor{darkblue}{RGB}{44,127,184}
\definecolor{magentaCB}{RGB}{221,28,119}
\definecolor{morange}{RGB}{255, 187, 0}
\definecolor{mblue}{RGB}{ 0, 161, 241}

\newtheorem{theorem}{Theorem}

\newtheorem{definition}{Definition}

\newcommand{\mbf}{\mathbf}

\newcommand{\ip}[2]{\left\langle {#1},\ {#2} \right\rangle}

\newcommand{\norm}[1]{\Vert {#1} \Vert}

\DeclareMathOperator{\dist}{dist}

\newcommand{\fnorm}[1]{\norm{#1}_{\mathrm{F}}}

\DeclareMathOperator{\Lip}{Lip}
\DeclareMathOperator{\prox}{prox}
\DeclareMathOperator{\proj}{proj}

\usepackage{algorithm}
\usepackage{algpseudocode}
\newcommand*\Let[2]{\State #1 $\gets$ #2}
\algrenewcommand\algorithmicrequire{\textbf{Input:}}
\algrenewcommand\algorithmicensure{\textbf{Output:}}

\usepackage[verbose=true,letterpaper]{geometry}
\AtBeginDocument{
	\newgeometry{
		textheight=8.8in,
		textwidth=6.3in,
		top=1in,
		headheight=14pt,
		headsep=25pt,
		footskip=30pt
	}
}

\title{Sparse Principal Component Analysis via Variable Projection}

\date{} 					

\author{%
	N. Benjamin Erichson \\
	UC Berkeley\\
	\texttt{erichson@berkeley.edu}
	\And
	Peng Zheng \\
	University of Washington \\
	\texttt{zhengp@uw.edu}
	\And
	Krithika~Manohar \\
	University of Washington \\
	\texttt{kmanohar@uw.edu}
	\And
	Steven L. Brunton \\
	University of Washington \\
	\texttt{sbrunton@uw.edu}
	\And
	J. Nathan Kutz \\
	University of Washington \\
	\texttt{kutz@uw.edu}
	\And	
	Aleksandr~Y.~Aravkin \\
	University of Washington\\
	\texttt{saravkin@uw.edu}
}

\begin{document}

	\maketitle
	\begin{abstract}
	Sparse principal component analysis (SPCA) has emerged as a powerful technique for modern data analysis, providing improved interpretation of low-rank structures by identifying localized spatial structures in the data and disambiguating between distinct time scales.
	We demonstrate a robust and scalable SPCA algorithm by formulating it as a value-function optimization problem. 
	This viewpoint leads to a flexible and computationally efficient algorithm.  Further, we can leverage randomized methods from linear 
	algebra to extend the approach to the large-scale (big data) setting. 
	Our proposed innovation also allows for a robust SPCA formulation which obtains meaningful sparse principal components in spite of grossly corrupted input data. 
	The proposed algorithms are demonstrated using both synthetic and real world data, and show exceptional computational efficiency and diagnostic performance.
	\end{abstract}

\section{Introduction}
\label{Introduction}

A wide range of  phenomena in the physical, engineering, biological, and social sciences feature rich dynamics that give rise to multiscale structures in both space and time, including fluid dynamics, atmospheric-ocean interactions, climate modeling, epidemiology, and neuroscience.
Remarkably, the underlying dynamics of such systems are typically inherently low-rank in nature, generating data sets where dimensionality reduction techniques, such as principal component analysis (PCA), can be used as a critically enabling diagnostic tool for interpretable characterizations of the dynamics. 
PCA decompositions express time-varying patterns as a linear combination of the dominant correlated spatial activity of the state of a system as it evolves in time.
Although commonly used, the PCA approach generates global modes that often mix or blend various spatio-temporal scales, and cannot identify underlying governing dynamics that act at separate scales.  
Moreover, classic PCA also tends to overfit data where the number of observations is smaller than the number of variables~\cite{cai2013}.  

Constrained or regularized matrix decompositions provide a more flexible approach for modeling dynamic patterns. 
Specifically, prior information can be introduced through sparsity promoting regularizers to obtain a more parsimonious approximation of the data which typically provides improved interpretability. 
Among others, {\em sparse principal component analysis} (SPCA) has emerged as a popular and powerful technique for modern data analysis. 
SPCA promotes sparsity in the modes, i.e., the sparse modes have only a few {\em active} coefficients, while the majority of coefficients are constrained to be zero.
The resulting sparse modes are often highly localized and more interpretable than the global PCA modes obtained from traditional PCA.
As a consequence, sparse regularization of PCA allows for a decomposition strategy that can specifically identify localized spatial structures in the data and disambiguate between distinct time scales, both of which are ubiquitous in measurement data of complex systems.
This is exemplified by many physical phenomena including the El Ni\~no warming event, which is characterized by a localized warm temperature profile which traverses the southern Pacific ocean. This is a highly localized mode that, as will be shown, is well characterized by SPCA, while standard PCA yields a global mode with nonzero values across the entire globe.

While the idea of sparsifying the weight vectors is not new, simple ad-hoc techniques such as naive thresholding can lead to misleading results.
A formal approach to SPCA, using $\ell_1$ regularization, was first proposed by Jolliffe et al.~\cite{jolliffe2003modified}. 
This pioneering work lead to a variety of sparsity promoting algorithms~\cite{zou2006sparse,d2005direct,d2008optimal,sigg2008expectation,shen2008sparse,witten2009penalized,journee2010generalized}. 
The success of sparse PCA in obtaining interpretable modes motivates the general approach developed in this paper. 
Specifically, our method offers three immediate improvements over previously proposed SPCA algorithms: (1) a faster and more scalable algorithm, (2) robustness to outliers, and (3) straightforward extension to nonconvex regularizers, including the $\ell_0$ norm.  
Scalability is essential for many applications --- for example, dynamical systems generate very large-scale datasets, such as the sea surface temperature 
data analyzed in this paper. Robust formulations allow SPCA to be deployed in a broader setting, 
where data contamination could otherwise hide sparse modes. Nonconvex regularizers are not currently available in SPCA software --- we show 
that the modes we get with these approaches are better in synthetic examples, and more interpretable for real data.

\paragraph{Contributions of this work}
In this work, we develop a scalable and robust approach for SPCA. 
A key feature of the approach is the use of {\it variable projection} to partially minimize 
over orthogonally constrained variables. This idea was used in the original alternating approach of~\cite{zou2006sparse}, and we innovate upon this idea by recasting the problem as a value-function optimization.  
This viewpoint allows for significantly faster algorithms, scalability, and broader applicability.
We also allow nonconvex regularization on the loadings, which further improves interpretability and sparsity. 
Not only does the method scale well, but it is further accelerated using randomized methods for linear algebra~\cite{erichson2016randomized}.
Further, the proposed approach extends to robust SPCA formulations, which can obtain meaningful principal components even with grossly corrupted input data.
The outliers are modeled as perturbations to the data, as in the robust PCA model~\cite{candes2011robust,bouwmans2017decomposition,aravkin2016dual}.
These innovations provide a flexible and highly-efficient algorithm for modern data analysis and diagnostics that enables a wide range of critical applications at a scale not previously possible with other leading algorithms.

\paragraph{Organization}
The manuscript is organized as follows: Section~\ref{sec:background} reviews PCA and the variable projection framework. 
Section~\ref{sec:overview} provides a detailed problem formulation and discusses the variable projection viewpoint which is advocated in this paper. Further, different loss functions and regularizes are discussed.
We present the details of the proposed algorithms in Section~\ref{sec:algorithm}. First, the standard case, using the least squares loss function, is discussed. Next, a randomized acceleration, and a robust variant of the method are presented. 
The method is applied to several examples in Section~\ref{sec:results}, where
SPCA correctly identifies dynamics occurring at different timescales in multiscale data. 
We draw conclusions about the method and discuss its outlook in Section~\ref{sec:discussion}. 

\paragraph{Notation}
Scalars are denoted by lower case letters $x$, and vectors in $\mathbb{R}^{n}$ are denoted as bold lower case letters $\mathbf{x}=[x_1,x_2,\dots, x_n]^\top$.
Matrices are denoted by bold capital letters $\mathbf{M}$. 
The transpose of a real matrix is denoted as $\mathbf{M}^\top$.
The spectral or operator norm of a matrix is denoted as $\|\cdot\|$ and the 
Frobenius norm is denoted as $\|\cdot\|_F$.

\section{Background}\label{sec:background}

\subsection{Principal Component Analysis}

Principal component analysis (PCA) is a ubiquitous dimension reduction technique, tracing back to Pearson~\cite{pearson1901liii} and Hotelling~\cite{hotelling1933analysis}.
The aim of PCA is to find a set of new uncorrelated variables, called principal components (PCs), such that the first PC accounts for the greatest amount of variance in the data, the second PC for the second greatest variance, and so on.
More concretely, let $\mathbf{X}$ be a real data matrix of dimension $n\times p$, with column-wise zero empirical mean. The $n$ rows represent observations and the $p$ columns correspond to measurements of variables. The principal components $\mathbf{z}_i \in \mathbb{R}^{n}$ are formed as a linear weighted combination of the variables
\begin{equation} \label{eq:pci}
	\mathbf{z}_i = \mathbf{X} \mathbf{a}_i, 
\end{equation}
where $\mathbf{a}_i \in \mathbb{R}^{p}$ is a vector of weights.
This can be expressed more concisely as 
\begin{equation}
	\mathbf{Z} = \mathbf{X} \mathbf{A},
\end{equation}
with $\mathbf{Z}=[\mathbf{z}_1,\mathbf{z}_2,\dots,\mathbf{z}_p] \in \mathbb{R}^{n \times p}$ and $\mathbf{A}=[\mathbf{a}_1,\mathbf{a}_2,\dots,\mathbf{a}_p] \in \mathbb{R}^{p \times p}$. 
The orthonormal matrix $\mathbf{W}$ rotates the data into a new space, where the principal components sequentially capture the maximum variability in the input data.
The columns of $\mathbf{A}$ are also often denoted as modes, basis functions, principal direction or loadings. 

Mathematically, a variance maximization problem can be formulated to find the weight vectors $\mathbf{a}_i$. 
Alternatively, the problem can be formulated as a least-squares problem, i.e., minimizing the sum of squared residual errors between the input and the projected data
\begin{equation}\label{eq:opt_pca}
	\begin{aligned}
		& \underset{\mathbf{A}}{\text{minimize}}
		& & f(\mathbf{A}) = \fnorm{\mathbf{X} - \mathbf{X}\mathbf{A}\mathbf{A^\top}}^2 \\
		& \text{subject to}
		& & \mathbf{A^\top}\mathbf{A}=\mathbf{I},
	\end{aligned}
\end{equation}
where PCA imposes orthogonality constraints on the weight matrix $\mathbf{A}$. Given the singular value decomposition (SVD) of the centered (standardized) input matrix $\mathbf{X}$
\begin{equation*}
	\mathbf{X} = \mathbf{U}\mathbf{\Sigma}\mathbf{V}^\top,
\end{equation*}
the minimizer of~\eqref{eq:opt_pca} is given by the right singular vectors $\mathbf{V}$, i.e., we can set $\mathbf{A} = \mathbf{V}$. Further, the principal components are the scaled left singular vectors $\mathbf{Z}=\mathbf{U} \mathbf{\Sigma}$, where the entries of the diagonal matrix $\mathbf{\Sigma}$ are the singular values.
In most applications, we are only interested in the first $k$ dominant PCs which account for most of the variability in the input data. Thus, PCA allows one to reduce the dimensionality from $p$ to $k$ by simply truncating the SVD. The dominant $k$ PCs can be used to visualize the data in low-dimensional space, and as features for clustering, classification and regression. 

We refer the reader to~\cite{jolliffe1986principal} for an extensive treatment of PCA and its mechanics. Many extensions such as Kernel PCA have been proposed to extend and overcome some of the shortcomings of PCA, see~\cite{cunningham2015linear} for an brief overview.

\subsection{Variable Projection}

Consider any objective of the form 
\begin{equation}
	\label{eq:genprob}
	\min_{\mathbf{A}, \mathbf{B}} \,\, g(\mathbf{A},\mathbf{B}). 
\end{equation}
A classic example is the nonlinear least squares problem 
\begin{equation}
	\label{eq:nls}
	\min_{\mathbf{A}, \mathbf{B}} \quad \frac{1}{2}\|\mathbf{X} - \mathbf{\Phi}(\mathbf{B}) \mathbf{A}\|^2.
\end{equation}
The term `variable projection'~\cite{Golub2003} originally arose from the fact that the least squares projection of $\mathbf{X}$ 
onto the range of $\mathbf{\Phi}(\mathbf{B})$ has a closed form solution, which is used explicitly in iterative methods to optimize 
for $\mathbf{B}$.  
More generally, the word `projection' is now associated with epigraphical projection~\cite{RW98}, or partial minimization. 
We can rewrite~\eqref{eq:genprob} as a value function optimization problem:
\begin{equation}
	\label{eq:valuevp}
	\min_{\mathbf{B}}  \left\{ v(\mathbf{B}) := \min_{\mathbf{A}} \,\, g(\mathbf{A}, \mathbf{B}) \right\}.
\end{equation}
In many cases, the function $v(\mathbf{B})$ has an explicit form. In the classic problem~\eqref{eq:nls}, we have 
\[
v(\mathbf{B}) =\frac{1}{2}\|\mathbf{X}(\mathbf{I} - \mathcal{P}_{\mathcal{R}(\mathbf{\Phi}(\mathbf{B}))})\|^2 = \text{dist}^2(\mathbf{X} | \mathcal{R}(\mathbf{\Phi}(\mathbf{B}))),
\]
where $\mathcal{P}_{\mathcal{R}(\mathbf{\Phi}(\mathbf{B}))}$ is a projector on the range of $\mathbf{\Phi}(\mathbf{B})$. Explicit expressions are not necessary 
as long as we have an efficient routine to compute 
\[
\mathbf{A}(\mathbf{B}) = \arg\min_{\mathbf{A}} \,\, g(\mathbf{A}, \mathbf{B}). 
\] 
For many problems, we can find first and second derivatives of $v(\mathbf{B})$. For example, when $g$ is smooth and $\mathbf{A}(\mathbf{B})$ is unique, 
we have 
\[
\nabla v(\mathbf{B}) = \partial_{\mathbf{B}} g(\cdot, \cdot) |_{(\mathbf{A}(\mathbf{B}), \mathbf{B})}.
\]
Formulas for second derivatives are collected in~\cite{aravkin2017efficient}. 
Variable projection was recently used to solve a range of large-scale structured problems in PDE-constrained optimization, nuisance parameter estimation, 
exponential fitting, and optimized dynamic mode decomposition~\cite{Aravkin2012c,aravkin2017efficient,hong2017revisiting,askham2018variable}. 

\section{Problem Formulation for Sparse Principal Component Analysis (SPCA)}\label{sec:overview}
Sparse PCA aims to find a set of sparse weight vectors, i.e., weight vectors with only a few `active' (nonzero) values. 
In this manuscript, we build on the seminal work by Zou, Hastie and Tibshirani~\cite{zou2006sparse}, who treat SPCA as a regularized regression problem. 
More concretely, their formulation directly incorporates sparsity inducing regularizers into the optimization problem: 
\begin{equation}\label{eq:spca_obj}
	\begin{aligned}
		\underset{\mathbf{A,B}}{\text{minimize}}~~
		& f(\mathbf{A,B}) = \tfrac{1}{2}\fnorm{\mathbf{X} - \mathbf{X}\mathbf{B}\mathbf{A^\top}}^2 + \psi(\mbf B) \\
		\text{subject to}~~
		& \mathbf{A^\top}\mathbf{A} = \mathbf{I},
	\end{aligned}
\end{equation}
where $\mathbf{B}$ is a sparse weight matrix and $\mathbf{A}$ is an orthonormal matrix. The penalty $\psi$ denotes a sparsity inducing regularizer such as the LASSO ($\ell_1$ norm) or the elastic net (a combination of the $\ell_1$ and squared $\ell_2$ norm).
The optimization problem is minimized using an alternating algorithm: 
\begin{itemize}
	\item \textbf{Update $\mathbf{A}$.} With $\mathbf{B}$ fixed, we find an orthonormal matrix $\mathbf{A^\top}\mathbf{A} = \mathbf{I}$ 
	which minimizes 
	\[
	\|\mathbf{X} - \mathbf{X}\mathbf{B}\mathbf{A^\top} \|_F^2.
	\]
	This is the orthogonal Procrustes problem~\cite{gower2004procrustes} (see Appendix~\ref{sec:Procrustes}), which has  
	a closed form solution $\mathbf{A}^* = \mathbf{U} \mathbf{V}^\top$, where 
	\(
	\mbf X^\top\mbf X\mbf B = \mbf U \mbf \Sigma \mbf V^\top.
	\)
	\item[]
	\item \textbf{Update $\mathbf{B}$.} With $\mbf A$ fixed, we solve the optimization problem 
	\[
	\underset{\mathbf{B}}{\min}~~\tfrac{1}{2}\fnorm{\mathbf{X} - \mathbf{X}\mathbf{B}\mathbf{A^\top}}^2 + \psi(\mbf B).
	\]
	The problem splits across the $k$ columns of $\mathbf{B}$, yielding a regularized regression problem in each case:  
	\[
	\mathbf{b}^*_j = \argmin_{\mathbf{b}_j} \,\, \tfrac{1}{2}\norm{\mathbf{X} \mathbf{A}(:,j) - \mathbf{X}\mathbf{b}_j}^2 + \psi(\mathbf{b}_j).
	\]
\end{itemize}
The principal components are then formed as a sparsely weighted linear combination of the observed variables $\mathbf{Z} = \mathbf{X}\mathbf{B}$.
The data can be approximately rotated back as $\mathbf{\widetilde{X}} = \mathbf{Z}\mathbf{A}^\top$.

Coordinate descent or least angle regression (LARS) are used to solve each of the $k$ subproblems ~\cite{efron2004least}. 
The $\mbf B$ update relies on solving a strongly convex problem, and in particular the update is unique, and the algorithm as described converges to a stationary point by the analysis of~\cite{tseng2001convergence}. 
Replacing $\psi$ with a nonconvex regularizer, such as $\psi(\mbf B) = \alpha \|\mbf B\|_0 + \beta \|\mbf B\|^2$, makes it difficult to guarantee anything about the $\mbf B$ update. However, as we show, using the value function~\eqref{eq:valuevp} from the variable projection viewpoint yields an efficient implementation and a straightforward convergence analysis. 

\subsection{Variable Projection Viewpoint}
The $\mbf A$ update in the method of~\cite{zou2006sparse} is in closed form, 
while the $\mbf B$ update requires an iterative method. 
To exploit the efficiency of the $\mbf A$ update, we consider projecting out $\mbf A$ and introducing the sparse PCA value function
\[
v(\mbf B) := \min_{\mbf A} \quad \tfrac{1}{2}\fnorm{\mathbf{X} - \mathbf{X}\mathbf{B}\mathbf{A^\top}}^2 \quad 
\text{subject to}~~
\mathbf{A^\top}\mathbf{A} = \mathbf{I},
\] 
viewing the original SPCA problem~\eqref{eq:spca_obj} as 
\begin{equation}
	\label{eq:valueProb}
	\min_{\mbf B} \,\, v(\mbf B) + \psi(\mbf B). 
\end{equation}
We show that $v(\mbf B)$ is differentiable with a Lipschitz continuous gradient, and derive its explicit form. 
This viewpoint permits the use of any desired proximal (prox) algorithm to  minimize~\eqref{eq:valueProb}, including proximal gradient (see e.g.,~\cite{parikh2014proximal}) and FISTA~\cite{beck2009fast}, with the caveat that an $\mbf A$ update is computed 
every time $v(\mbf B)$ is evaluated. For the original SPCA problem, this approach rebalances the work between 
the $\mbf A$ and $\mbf B$ updates, using a single operator to update $\mbf B$ instead of an iterative routine. 
When combined with randomized techniques for computing the $\mbf A$ update, we get  
an order of magnitude acceleration compared to current SPCA software.

The variable projection viewpoint~\eqref{eq:valueProb} also allows a robust SPCA approach with  the Huber loss function. 
Simply replacing the quadratic penalty in~\eqref{eq:spca_obj} with a different loss would destroy the efficient structure of the $\mbf A$ update, requiring an iterative routine to solve for it. 
Instead, we use a special characterization of the Huber function to obtain a formulation with three rather than two variables, preserving the efficiency of each update. We extend our analysis to this case, so the robust formulation can also be used with any prox-friendly $\psi$ regularizer, including the nonconvex 
example discussed above. 

In summary, the value function viewpoint also makes it easy to extend to a broader problem setting, and we consider the following objective:
\begin{equation}\label{eq:general_spca_obj}
	\min_{\mbf A, \mbf B} \,\, f(\mbf A, \mbf B):= \rho(\mbf X - \mbf X \mbf B \mbf A^\top) + \psi(\mbf B)  \quad \text{subject to}~~\mathbf{A^\top}\mathbf{A} = \mathbf{I},
\end{equation}
where $\rho$ is a separable loss, while $\psi$ is a separable regularizer for $\mbf B$. 

\subsection{Regularizers for Sparsity}
The SPCA framework incorporates a range of sparsity-inducing regularizers $\psi$. 
Sparsity is achieved by introducing additional information into the model to find the most meaningful `active' (non-zero) entries in $\mathbf{B}$, while most of the loadings are constrained to be zero. 
Sparse approaches work well when many variables are redundant, i.e., not required to capture the underlying coherent model structure. 
Regularization also prevents overfitting and provides a path to solve ill-posed problems, which are frequently encountered in the analysis of high-dimensional datasets.

Appendix~\ref{app:regularizers} provides a brief discussion of some popular regularizers to promote sparsity.

\section{Fast Algorithms for Sparse PCA}\label{sec:algorithm}
\subsection{Sparse PCA via Variable Projection}
As a standard problem, we discuss the variable projection algorithm for~\eqref{eq:general_spca_obj} using the least squares loss function.
We partially minimize in $\mbf A$ to obtain the {\it value} function
\begin{equation}
	\label{eq:value}
	v(\mbf B) := \min_{\mbf A} \frac{1}{2} \,\, \|\mbf X - \mbf X \mbf B \mbf A^\top\|_F^2   \quad \text{subject to}~~\mathbf{A^\top}\mathbf{A} = \mathbf{I}.
\end{equation}
Evaluating this value function given $\mathbf{B}$ reduces to solving the 
orthogonal Procrustes problem~\cite{gower2004procrustes}, with closed form solution 
\begin{equation}\label{eq:solA}
	\mathbf{A}(\mathbf{B}) = \mbf U \mbf V^\top,
\end{equation}
where $\mbf U$ and $\mbf V$ are the left and right singular vectors of
\(
\mbf X^\top\mbf X\mbf B = \mbf U \mbf \Sigma \mbf V^\top,
\)
see Appendix~\ref{sec:Procrustes}.
Variable projection takes advantage of this closed form solution.
Partially minimizing in $\mbf A$ via the SVD has additional advantages over using an iterative algorithm when $\mbf A$  is ill-conditioned. This is an important consideration for robust penalties, 
where a closed form solution for $\mbf A$ is not immediately available, see Section~\ref{sec:robustSPCA}.

The SPCA problem~\eqref{eq:general_spca_obj} is nonconvex, and so is the value function $v(\mbf B)$. 
To better understand $v(\mbf B)$, we consider the following simple two-dimensional example
\[
f(\bm a, \bm b) = \frac{1}{2}\fnorm{\mbf X - \mbf X \bm b \bm a^\top}^2  \quad \text{subject to}~~\mathbf{a^\top}\mathbf{a} = \mathbf{1},
\]
where $\mbf X\in\R^{2\times2}$, $\bm a, \bm b\in\R^2$. 
We write  $v(\mbf b): \mathbb{R}^{2}\rightarrow \mathbb{R}$ explicitly as 
\[
v(\bm b) = \frac{1}{2}\fnorm{\mbf X - \mbf X \bm b\,{\bm a}(\bm b)^\top}^2, \quad
{\bm a}(\bm b) = \frac{\mbf X^\top \mbf X \bm b}{\|\mbf X^\top \mbf X \bm b\|}.
\]
Figure~\ref{fig:level_set} shows the level sets of this function, which are clearly nonconvex. We also see that $v(\mbf b)$ 
is smooth except at $\mbf b = 0$.
\begin{figure}[!t]
	\centering
	\begin{overpic}[width=0.8\textwidth]{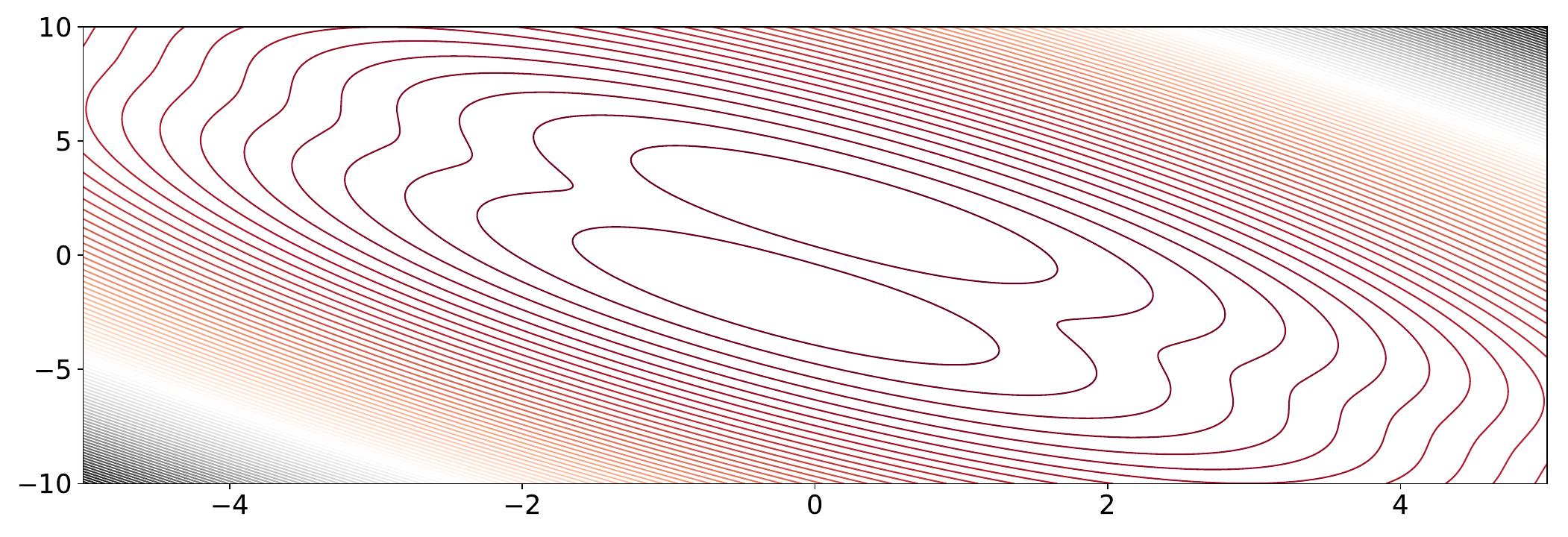} 
		\put(-2,17){\rotatebox[]{90}{coordinate $a(b)$}}
		\put(43,-2){{coordinate b}}
	\end{overpic}\vspace{+0.2cm}
	\caption{Level set of simple 2D projected function.}
	\label{fig:level_set}
\end{figure}

We apply proximal gradient methods (see e.g.,~\cite{parikh2014proximal}) to find a stationary point of the value function $v(\mbf b)$~\eqref{eq:value}.
It is easy to both evaluate $v(\mbf b)$ and to compute the gradient.
We obtain $\mbf a(\mbf b)$ using~\eqref{eq:solA} and then use the formula 
\[
\nabla v(\mathbf{b}) = \nabla_{\mathbf{b}} f(\mathbf{a}, \mathbf{b})|_{\mathbf{a} = \mathbf{a}(\mathbf{b})} = \mbf X^\top(\mbf X - \mbf X \bm b\,{\bm a}(\bm b)^\top){\bm a}({\bm b}) .
\]
This yields a simple and efficient algorithm detailed in Algorithm~\ref{alg:vp_f}.

Note, that this algorithm is very similar to the proximal alternating minimization (PAM) method of~\cite{attouch2010proximal} and the proximal alternating linearized minimization (PALM) method by~\cite{bolte2014proximal}. 
In fact, our proposed algorithim can be thought of as a limiting case of PAM,  where the proximal term for block $\bm A$ is completely 
ignored, i.e., the update is done without reference to the current iterate $\bm A_k$. While the convergence theory 
of~\cite{attouch2010proximal} does not cover this limiting case, the convergence result we present in Appendix~\ref{sec:proof_convergence}
uses the variable projection strategy to derive a standalone analysis for this case, with the result summarized in Theorem~\ref{them:convergence}.
In particular, we show a $1/N$ rate of convergence of a simple stationarity criterion. In contrast, the rates of~\cite{attouch2010proximal} are only 
known to activate eventually -- a weaker result because they consider a much broader problem class.

Since the objective~\eqref{eq:general_spca_obj} is nonconvex, our convergence analysis targets its stationary points.
\begin{definition}[Stationary Points]
	Assume that $\rho$ is smooth, we call a pair $(\mbf A, \mbf B)$ a {\it stationary point} when it satisfies
	\begin{align*}
		\bm 0 &\in \nabla \rho(\mbf A \mbf X^\top \mbf B^\top - \mbf X^\top)\mbf B\mbf X + \partial \varphi(\mbf A),\\
		\bm 0 &\in \mbf X^\top \nabla\rho(\mbf X\mbf B\mbf A^\top - \mbf X) \mbf A + \partial \psi(\mbf B),
	\end{align*}
	where $\partial \psi$ is the limiting subdifferential defined in Section~\ref{sec:notation}.
\end{definition}
The following theorem provides a sublinear convergence guarantee for Algorithm~\ref{alg:vp_f}.
The convergence is stated in terms of the non-stationarity  criterion  $T$ defined by :
\begin{equation}\label{eq:cond}
	\begin{aligned}
		T(\mbf A, \mbf B)   = & \min \{\tfrac{1}{2}\fnorm{\mbf U}^2 + \tfrac{1}{2}\fnorm{\mbf V}^2:\\
		&\mbf U \in \nabla \rho(\mbf A \mbf X^\top \mbf B^\top - \mbf X^\top)\mbf B\mbf X + \partial \varphi(\mbf A),\\
		& \mbf V \in \mbf X^\top \nabla\rho(\mbf X\mbf B\mbf A^\top - \mbf X) \mbf A + \partial \psi(\mbf B).\}
	\end{aligned}
\end{equation}

\begin{theorem}[Convergence of \ref{alg:vp_f}]\label{them:convergence}
	Assume $\rho = \rho_\mathrm{F}$, then the optimality criterion satisfies
	\[
	\min_{1\le k\le N} T(\mbf A_k, \mbf B_k) \le \frac{2(\|\mbf X\|_2^2 + L)^2}{N\|\mbf X\|_2^2} f(\mbf A_1, \mbf B_1),
	\]
	where $L$ is the Lipchitz constant for $\partial \psi$.
\end{theorem}
See Appendix~\ref{sec:proof_convergence} for the proof.

\begin{algorithm}[!b]
	\caption{Variable projected proximal gradient method for~\eqref{eq:value}. We use the left singular vectors $\mbf V$ of the data matrix $\mbf X$ to initialize the factor matrices $\mbf A_0 = \mbf V$ and $\mbf B_0 = \mbf V$.}
	\scalebox{0.99}{		
		\begin{minipage}{150mm}
			\begin{algorithmic}[1]
				\Require{$\mbf A_0$, $\mbf B_0$, $\mbf X$, $k=0$, $\epsilon >0$, $\gamma = 1/\|\mbf X\|_2^2$}
				\While{$T(\mbf A_{k}, \mbf B_k) \geq \epsilon$} \Comment{See~\eqref{eq:cond}}
				\Let{$\mbf B_{k+1}$}{$\prox_{\gamma r}(\mbf B_k - \gamma \mbf X^\top(\mbf X \mbf B_k - \mbf X \mbf A_k))$}\Comment{See~\eqref{eq:prox}}
				\Let{$\mbf A_{k+1}$}{$\argmin_{\mbf A} \frac{1}{2}\|\mbf X - \mbf X \mbf B_{k+1} \mbf A^\top\|_\mathrm{F}^2  \quad \text{subject to}~~\mathbf{A^\top}\mathbf{A} = \mathbf{I}$}
				\Comment{See~\eqref{eq:solA}}
				\Let{$k$}{$k+1$}
				\EndWhile
				\Ensure{$\mbf A_{k+1}$, $\mbf B_{k+1}$}
			\end{algorithmic}
	\end{minipage}}
	\label{alg:vp_f}
\end{algorithm}

\subsection{Randomized Sparse PCA}\label{sec:randomizedSPCA}

Low-rank matrices are pervasive in data science~\cite{udell2017nice}. Indeed, the working assumption of dimension reduction techniques such as PCA and SPCA is that the data matrix contains redundant information, i.e., has low-rank structure. If a data matrix features low-rank structure, then randomized methods for linear algebra allow the efficient computation of low-rank approximations such as the SVD and PCA~\cite{halko2011finding,mahoney2011randomized,drineas2016randnla,erichson2016randomized}.

Randomized methods construct a low-dimensional sketch (representation) of the data, which aims to capture the essential information of the original data. 
Using this idea, we can reformulate~\eqref{eq:general_spca_obj} as a randomized value function which takes the form
\begin{equation}\label{eq:randomized_spca_obj}
	v(\mbf B) := \min_{\mbf A} \,\,  \frac{1}{2}\|\widetilde{\mathbf{X}} - \widetilde{\mathbf{X}} \mbf B \mbf A^\top\|_F^2   \quad \text{subject to}~~\mathbf{A^\top}\mathbf{A} = \mathbf{I},
\end{equation}
where $\widetilde{\mathbf{X}} \in \mathbb{R}^{l\times p}$ denotes the sketch of $\mathbf{X} \in \mathbb{R}^{n\times p}$. Here, the dimension $l$ is chosen to be slightly larger than the target-rank $k$.
We proceed by forming a sample matrix $\mathbf{Y} \in \mathbb{R}^{n\times l}$:
\begin{equation}
	\mathbf{Y} = \mathbf{X}\mathbf{\Omega},
\end{equation}
where $\mathbf{\Omega} \in \mathbb{R}^{p\times l}$ is a randomly generated test matrix~\cite{halko2011finding}. 
Next, an orthonormal basis matrix is obtained by computing the QR-decomposition of the samples matrix $\mathbf{Y}=\mathbf{Q}\mathbf{R}$. Finally, the sketch is formed by projecting the input matrix to the range of $\mathbf{Y}$, which is low-dimensional:
\begin{equation}
	\widetilde{\mathbf{X}} = \mathbf{Q}^\top \mathbf{X}.
\end{equation}
We perform the projection step only once in order to initialize the (randomized) SPCA algorithm. In other words, this approach can be viewed as a pre-conditioning step, which is especially useful if the data matrix $\mathbf{X}$ is too big to fit into fast memory. 
This approach is suitable for input matrices with low-rank structure. The computational advantage becomes significant when the intrinsic  rank of the data is relatively small compared to the dimension of the ambient measurement space.
The quality of the sketch can be improved by computing additional power iterations~\cite{halko2011finding,erichson2016randomized}, especially if the singular value spectrum of $\mathbf{X}$ is only slowly decaying. We recommend computing at least two power iterations by default.
We refer the reader to seminal work by Halko et al.~\cite{halko2011finding}, for many more details on the randomized framework and a rigorous analysis of its performance  for low-rank approximations.

\subsection{Robust Sparse PCA via Variable Projection}\label{sec:robustSPCA}

Classically, SPCA is formulated as a least-squares problem, however, it is well-known that the squared loss is sensitive to outliers. 
In many real world situations we face the challenge that data are grossly corrupted due to measurement errors or other effects. This motivates the need of robust methods which can more effectively account for corrupt or missing data.
Indeed, several authors have proposed a robust formulation of SPCA, using the $\ell_1$ norm as a robust loss function, to deal with grossly corrupted data~\cite{MENG2012487,croux2013robust,hubert2016sparse}. 

For a robust formulation of SPCA, we use a closely related idea of separating a data matrix into a low-rank model and a sparse model. The architecture is depicted in Figure~\ref{fig:sparse_illustration1}.
\begin{figure}[!b]
	\vspace{-0.5in}
	\centering
	\begin{center}
		
		\begin{tikzpicture}[scale=0.75, transform shape]
			\begin{scope}[every node/.style={circle,thick,draw}]
				
				\node[draw=none, align=center] at (-4,   4.) {external world}; 
				\node[draw=none, align=center] at (3,   4.) {low-rank model};   
				\node[draw=none, align=center] at (7,   4.) {sparse model};

				\node[fill=darkred!0, draw=none, align=center] (T1) at (-4, -4.2) {observable\\ variables};            
				\node[fill=red!0, draw=none, align=center] (T2) at (-0.3, -4.2) {sparse basis\\ functions};        
				\node[fill=red!0, draw=none, align=center] (T3) at (3,   -4.2) {principal\\ components};            
				
				\node[fill=red!0, draw=none, align=center] (T3) at (7,   -4.2) {outliers};            			
				
				\node[fill=darkred!10] (A) at 			(-4.0,  3) {$x_1$};
				\node[fill=darkred!10] (B) at 			(-4.0,  1) {$x_2$};
				
				\node[fill=darkred!10] (D) at 			(-4.0, -1) {$x_{j}$};    
				\node[fill=darkred!10] (E) at 			(-4.0, -3) {$x_p$};        
				
				\node[fill=red!0, draw=none] (C1) at (-4.0,  0) {$\vdots$};
				\node[fill=red!0, draw=none] (C2) at (-4.0,  -2) {$\vdots$};

				\node[fill=blue!10] (F) at (3,  2.0) {$z_1$};
				
				\node[fill=blue!10] (H) at (3, 0.0) {$z_j$} ;   
				
				\node[fill=blue!10] (I) at (3, -2.0) {$z_k$} ;   
				
				\node[fill=darkred!0, draw=none] (C3) at (3,  1) {$\vdots$};
				\node[fill=darkred!0, draw=none] (C4) at (3,  -1) {$\vdots$};

				\node[fill=green!10] (s1) at 			(7.0,  3) {$s_1$};
				\node[fill=green!10] (s2) at 			(7.0,  1.0) {$s_2$};
				
				\node[fill=green!10] (s3) at 			(7.0, -1.0) {$s_{j}$};    
				\node[fill=green!10] (s4) at 			(7.0, -3) {$s_p$};        
				
				\node[fill=green!0, draw=none] (S5) at (7.0,  0) {$\vdots$};
				\node[fill=green!0, draw=none] (S6) at (7.0,  -2) {$\vdots$};
				
				\node[fill=green!0, draw=none,] (s1p) at 			(5,  3) {$+$};
				\node[fill=green!0, draw=none,] (s2p) at 			(5,  1.0) {$+$};
				
				\node[fill=green!0, draw=none,] (s3p) at 			(5, -1.0) {$+$};    
				\node[fill=green!0, draw=none,] (s4p) at 			(5, -3) {$+$};

				\draw[rounded corners] (-5.3, -4.9) rectangle (-2.7, 3.6) {};
				
				\draw[rounded corners] (1.7, -4.9) rectangle (4.3, 3.6) {};    
				
				\draw[rounded corners] (5.7, -4.9) rectangle (8.3, 3.6) {};

			\end{scope}
			
			\begin{scope}[>={Stealth[black]},
				every node/.style={fill = white,circle}]
				
				\path [<-, very thick, gray] (F) edge  (B);    
				\path [<-, very thick, gray] (F) edge  (D);         
				
				\path [<-, very thick, gray] (H) edge  (A);    
				\path [<-, very thick, gray] (H) edge  (B);    
				\path [<-, very thick, gray] (H) edge  (E);

				\path [<-, very thick, gray] (I) edge  (A);    
				\path [<-, very thick, gray] (I) edge  (D);           
				
			\end{scope}
		\end{tikzpicture}
	\end{center}

	\caption{Robust SPCA combines a low-rank and sparse model to represent the observable variables. The low-rank model forms the principal components as a sparsely weighted linear combination of the observed variables. The sparse model extracts outliers in the data.}
	\label{fig:sparse_illustration1}
\end{figure}		
This form of additive decomposition is well-known as robust principal component analysis (RPCA), and its remarkable ability to separate high-dimensional matrices into low-rank and sparse component makes RPCA an invaluable tool for data science~\cite{candes2011robust,bouwmans2017decomposition,aravkin2016dual}.
Specifically, we suggest using the Huber loss function $\rho = \rho_\mathrm{H}$ rather than the $\ell_1$ norm as the data misfit.
The Huber norm overcomes some of the shortcomings for the Frobenius norm and can be used as a more robust measure of fit~\cite{huber2011robust,maronna2006robust}. We define the Huber loss function as
\begin{equation*}
	\begin{aligned}
		\rho_\mathrm{H}(x;\kappa) &= \begin{cases}
			\kappa|x| - \kappa^2/2, & |x| > \kappa\\
			x^2/2, & |x| \le \kappa
		\end{cases},\\
		\rho_\mathrm{H}(\mbf A; \kappa) &= \sum_{i,j}\rho_\mathrm{H}(\mbf A_{ij};\kappa).
	\end{aligned}
\end{equation*}
Figure~\ref{GLT-KF} illustrates the least squares and the Huber loss functions.
The Huber loss function grows at a linear rate for residuals outside the thresholding parameter $\kappa$, rather than quadratically.
Hence, the influence of large deviations on the parameters is reduced. This is consistent with using a heavy tail distribution to model measurement errors.

\begin{figure}[!t]
	\centering
	\begin{subfigure}[t]{0.45\textwidth}
		\begin{center}
			\begin{tikzpicture}
				\begin{axis}[
					thick,
					axis lines = middle,
					enlargelimits = true,
					width=.75\textwidth, height=3cm,
					xmin=-3,xmax=3,ymin=0,ymax=2,
					no markers,
					samples=50,
					axis lines*=left, 
					axis lines*=middle, 
					scale only axis,
					xtick={-1,1},
					xticklabels={},
					ytick={0},
					] 
					\addplot[darkblue, domain=-3:+3,densely dashed, line width=2.0pt,]{.5*x^2};
					\addplot[line width=2.0pt, darkred, domain=-3:-1]{abs(x)-.5};
					\addplot[line width=2.0pt, darkred, domain=-1:+1]{.5*x^2};
					\addplot[line width=2.0pt, darkred, domain=1:+3]{abs(x)-.5};
					
				\end{axis}
			\end{tikzpicture}
		\end{center}
		\caption{Loss functions. }
	\end{subfigure}
	~
	\begin{subfigure}[t]{0.45\textwidth}
		\begin{center}
			\begin{tikzpicture}
				\begin{axis}[
					thick,
					axis lines = middle,
					enlargelimits = true,    
					width=.75\textwidth, height=3cm,
					xmin=-3,xmax=3,ymin=-2, ymax=2,
					no markers,
					samples=50,
					axis lines*=left, 
					axis lines*=middle, 
					scale only axis,
					xtick={-1,1},
					xticklabels={},
					ytick={0},
					] 
					\addplot[darkblue, line width=2.0pt, domain=-3:3,densely dashed]{x};
					\addplot[darkred, line width=2.0pt, domain=-3:-1]{-1};
					\addplot[darkred, line width=2.0pt, domain=-1:+1]{x};
					\addplot[darkred, line width=2.0pt, domain=1:+3]{1};
				\end{axis}
			\end{tikzpicture}
		\end{center}
		\caption{Influence functions (first derivatives). }
	\end{subfigure}
	
	\caption{\label{fig:robust} Illustration of the least-squares loss (dashed blue) and Huber (solid red) loss functions in (a); the first derivatives in (b) can be viewed as influence functions of the residuals.}
	
	\label{GLT-KF}
\end{figure}

The Huber penalty can be characterized as the (scaled) Moreau envelope of the $\ell_1$ norm, see Section~\ref{sec:Moreau}:
\begin{equation}
	\label{eq:hubermoreau}
	\rho_\mathrm{H}(x;\kappa)  = \min_{s} \,\, \frac{1}{2} \|s-x\|^2 + \kappa\|s\|_1.
\end{equation}
This characterization explicitly extracts outliers $s$ as sparse perturbations to the data. It also makes it possible to develop efficient algorithms for the robust case. In general, our approach applies to any robust norm that can be characterized as the Moreau envelope of a separable penalty. 

A naive approach loses the closed form of $\mbf A$~\eqref{eq:solA}. 
To preserve the advantages of partial minimization, we must place the Huber loss on the Procrustean bed of the orthogonal Procrustes problem. 
We use the Moreau characterization~\eqref{eq:hubermoreau} to explicitly model sparse outliers using the variable $\mbf S$, and rewrite Eq.~\eqref{eq:general_spca_obj} as follows:
\begin{equation}
	\label{eq:obj_robust}
	\begin{aligned}
		\min_{\mbf A, \mbf B, \mbf S}~~f_\mathrm{H}(\mbf A, \mbf B, \mbf S) :=& \frac{1}{2}\|\mbf X - \mbf X \mbf B \mbf A^\top - \mbf S\|_\mathrm{F}^2
		+ \psi(\mbf B) + \kappa\|\mbf S\|_1  \quad \text{subject to}~~\mathbf{A^\top}\mathbf{A} = \mathbf{I}.
	\end{aligned}
\end{equation}
Now we can again use the orthogonal Procrustes approach~\eqref{eq:solA} and reduce~\eqref{eq:obj_robust} to minimizing the value function 
\begin{equation}
	\label{eq:obj_robust_vf}
	\begin{aligned}
		\min_{\mbf B, \mbf S} \,\,  v_\mathrm{H}(\mbf B, \mbf S) := &
		\frac{1}{2}\|\mbf X - \mbf X \mbf B {\mbf A}(\mbf B, \mbf S)^\top - \mbf S\|_\mathrm{F}^2
		+ \psi(\mbf B) + \kappa\|\mbf S\|_1, 
	\end{aligned}
\end{equation}
where $ {\mbf A}(\mbf B, \mbf S)$ is given by $\mbf U \mbf V^\top $ with 
\begin{equation}
	\label{eq:solAhuber}
	(\mbf X - \mbf S)^\top \mbf X \mbf B = \mbf U \mbf \Sigma \mbf V^\top.
\end{equation}
Problem~\eqref{eq:obj_robust_vf} has the same structure as~\eqref{eq:value} in the variables $(\mbf B, \mbf S)$, and we can easily modify the algorithm to account for the additional block, as detailed in Algorithm~\ref{alg:vp_huber}. The partial minimization of $\mbf S$ is  a prox evaluation of the $\ell_1$ norm, which is the soft thresholding operator, see Table~\ref{table:reg}.%

\begin{algorithm}[!t]
	\caption{Gauss-Seidel proximal gradient method for~\eqref{eq:obj_robust}}
	\scalebox{0.99}{		
		\begin{minipage}{150mm}
			\begin{algorithmic}[1]
				\Require{$\mbf A_0, \mbf B_0, \mbf S_0, k = 0$,  $\gamma = 1/\|\mbf X\|_2^2$}
				\While{not converged}
				\Let{$\mbf B_{k+1}$}{$\prox_{\gamma r}(\mbf B_k - \gamma \mbf X^\top(\mbf X \mbf B_k - \mbf X \mbf A_k + \mbf S_k \mbf A_k))$} \quad \Comment{See~\eqref{eq:obj_robust_vf}}
				\Let{$\mbf A_{k+1}$}{$\argmin_{\mbf A} \frac{1}{2}\|\mbf X - \mbf X \mbf B_{k+1} \mbf A^\top - \mbf S_k\|_\mathrm{F}^2 \quad \text{subject to}~~\mathbf{A^\top}\mathbf{A} = \mathbf{I}$}\quad\Comment{See~\eqref{eq:solAhuber}}
				\Let{$\mbf S_{k+1}$}{$\argmin_{\mbf S}\frac{1}{2}\|\mbf X - \mbf X \mbf B_{k+1} \mbf A_{k+1}^\top - \mbf S\|_\mathrm{F}^2 + \kappa \|\mbf S\|_1$} \quad\Comment{ See~\eqref{eq:obj_robust_vf}}
				\Let{$k$}{$k+1$}
				\EndWhile
				\Ensure{$\mbf A_k$, $\mbf B_k$, $\mbf S_k$}
			\end{algorithmic}
	\end{minipage}}
	\label{alg:vp_huber}
\end{algorithm}

\section{Spatiotemporal SPCA}

Sparse decompositions are becoming increasingly relevant for data-driven spatiotemporal analysis of physical systems.
The recent proliferation of machine learning and manifold learning methods seek interpretable models using physically meaningful constraints~\cite{Ozolicnvs2013pnas,schaeffer2013sparse,taira_nair_brunton_2016,loiseau2018constrained}.
However, standard orthogonal decompositions such as SVD or proper orthogonal decomposition (POD) may suffer from overfitting and the resulting spatial modes are spatially dense. By promoting sparsity in the modes, SPCA is able to yield modes that may be more interpretable.  

The goal of spatiotemporal modal analysis is a system decomposition that is separable in space and time,
\begin{equation}
	\mathbf{x}(t) = \sum_{j=1}^r a_j(t) \boldsymbol{\phi}_j,\label{Eq:POD}
\end{equation}
where $\boldsymbol{\phi}_j$ is a mode evaluated at a grid of spatial locations.
Although this basis is fixed in time, there are several methods that adapt the basis over time for enhanced performance, although often at an increased computational cost.  These adaptive basis approaches include the incremental SVD~\cite{brand2002incremental}, dynamically orthogonal modes~\cite{sapsis2009dynamically}, adaptive-h refinement~\cite{carlberg2015adaptive}, and optimally time-dependent (OTD) modes~\cite{farazmand2016dynamical,babaee2016minimization}.

Classical data-driven analysis seeks a low-rank approximation, as in \eqref{Eq:POD}, given a data matrix of snapshots in time
\begin{equation}
	\mathbf{X} = [\mathbf{x}(t_1)~ \mathbf{x}(t_2)~ \dots~ \mathbf{x}(t_p)].
\end{equation}
The proper orthogonal decomposition is a canonical data-driven decomposition in the analysis of high-dimensional flows, which seeks the optimal rank-$r$ orthogonal projection of the data that approximates the covariance of $\mathbf{X}$. 
The optimal low-rank projection is given by the dominant $k$ scaled principal components $\mathbf{Z}= \mathbf{X}\mathbf{V} = \mathbf{U}\mathbf{\Sigma}$, obtained from the singular value decomposition $\mathbf{X}=\mathbf{U}\boldsymbol{\Sigma}\mathbf{V}^T$.
We define the modes to be $\mathbf{\Phi} = \mathbf{U}$, resulting in the separable decomposition
\begin{equation}
	\mathbf{X} = \mathbf{\Phi}\mathbf{C^T},
\end{equation}
where $\mathbf{C^T} = \mathbf{\Sigma}\mathbf{V}^T$. While POD modes numerically approximate the data, they may not be physically meaningful. POD modes do not generally correspond to coherent structures that persist in time. Sparse PCA, on the other hand, imposes sparsity in the spatial modes while maintaining time independence. In our framework, spatial modes are given by
\begin{equation}
	\mathbf{\Phi} = \mathbf{X}\mathbf{B} = \begin{bmatrix}
		\vert & \vert &  &  \vert\\
		\mathbf{Xb}_1   & \mathbf{Xb}_2 & \dots & \mathbf{Xb}_k  \\
		\vert & \vert &  & \vert
	\end{bmatrix},
\end{equation}
which represents a sparse linear combination of the snapshots and sparse modes $\mathbf{\Phi}$. Recall, that the columns of $\mbf B$ are the sparse weight vectors. 
As we shall demonstrate, SPCA modes display greater correspondence to coherent structures in various flows.	

\section{Results}\label{sec:results}

We now apply our SPCA framework to a number of example systems of interest, ordered by increasing complexity. These examples capture many challenges that motivate the new algorithms. 
The first example is an artificial dataset with high-dimensional measurements and low-dimensional structures across multiple scales.  
In this example, there is a ground truth, providing a straightforward benchmark for SPCA and robust SPCA. 
The second example applies SPCA to a highly structured fluid flow, characterized by laminar vortex shedding behind a circular cylinder at low Reynolds number, which is a benchmark problem in fluid dynamics~\cite{Noack2003jfm}. 
Fluid flows are ideal for developing interpretable models of multiscale physics and deploying sparse sensors for estimation and control. This is because they are high-dimensional systems that often exhibit low-dimensional coherent patterns that are spatially localized~\cite{Brunton2015amr,Taira2017aiaa}.
The third example involves high-dimensional satellite data of the ocean surface temperature, a complex multiscale system that is intimately related to global circulation and climate.  
In all of these examples, the data is dynamic, high-dimensional, exhibits low-dimensional patterns at multiple scales, and has fewer snapshots in time than measurements in space.   
The proposed SPCA framework allows efficient computations on large systems, yields robust estimates from noisy data, and gives interpretable modes that can be used for the downstream tasks in dynamical systems modeling and control.

\subsection{Multiscale Video Example}

\begin{figure}[!b]
	\centering
	\centering
	\DeclareGraphicsExtensions{.pdf}
	\includegraphics[width=0.9\textwidth]{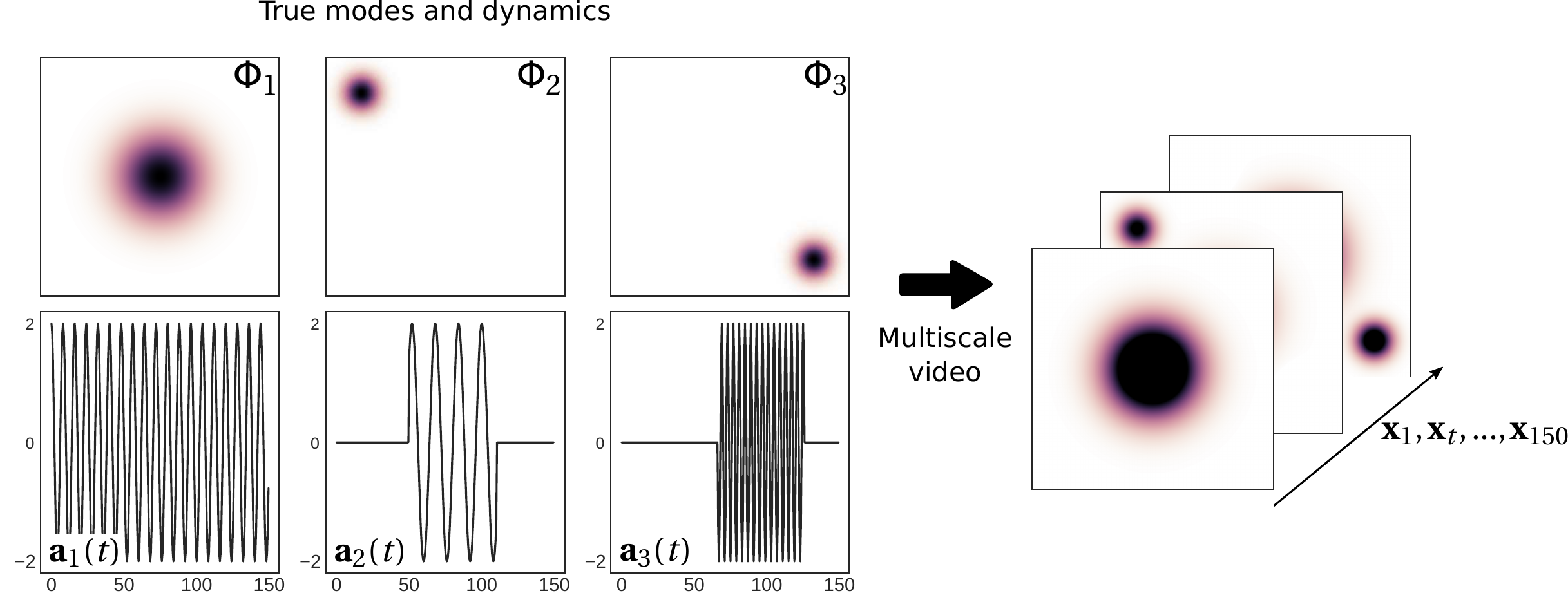}
	
	\caption{Multiscale video model. Each frame of this multiscale video is high-dimensional with $200\times 200$ pixels, however, the system has only three degrees of freedom.}
	\label{fig:multi_model}
\end{figure}

\begin{figure}[!b]
	\centering
	\begin{subfigure}[t]{0.38\textwidth}
		\centering
		\DeclareGraphicsExtensions{.pdf}
		\includegraphics[width=1\textwidth]{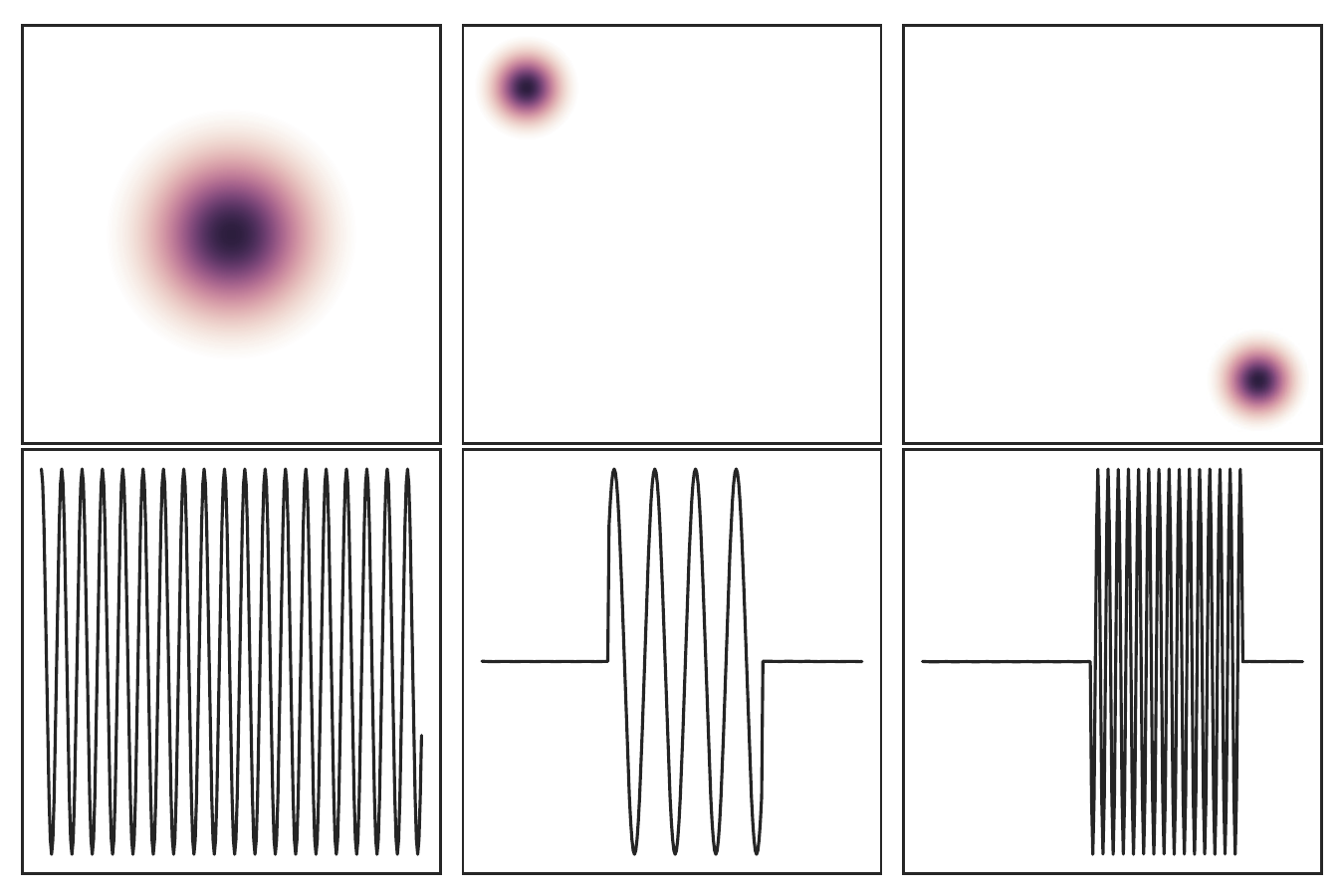}
		\caption{SPCA. }
		\label{fig:spca_modes_vid}
	\end{subfigure}
	~
	\begin{subfigure}[t]{0.38\textwidth}
		\centering
		\DeclareGraphicsExtensions{.pdf}
		\includegraphics[width=1\textwidth]{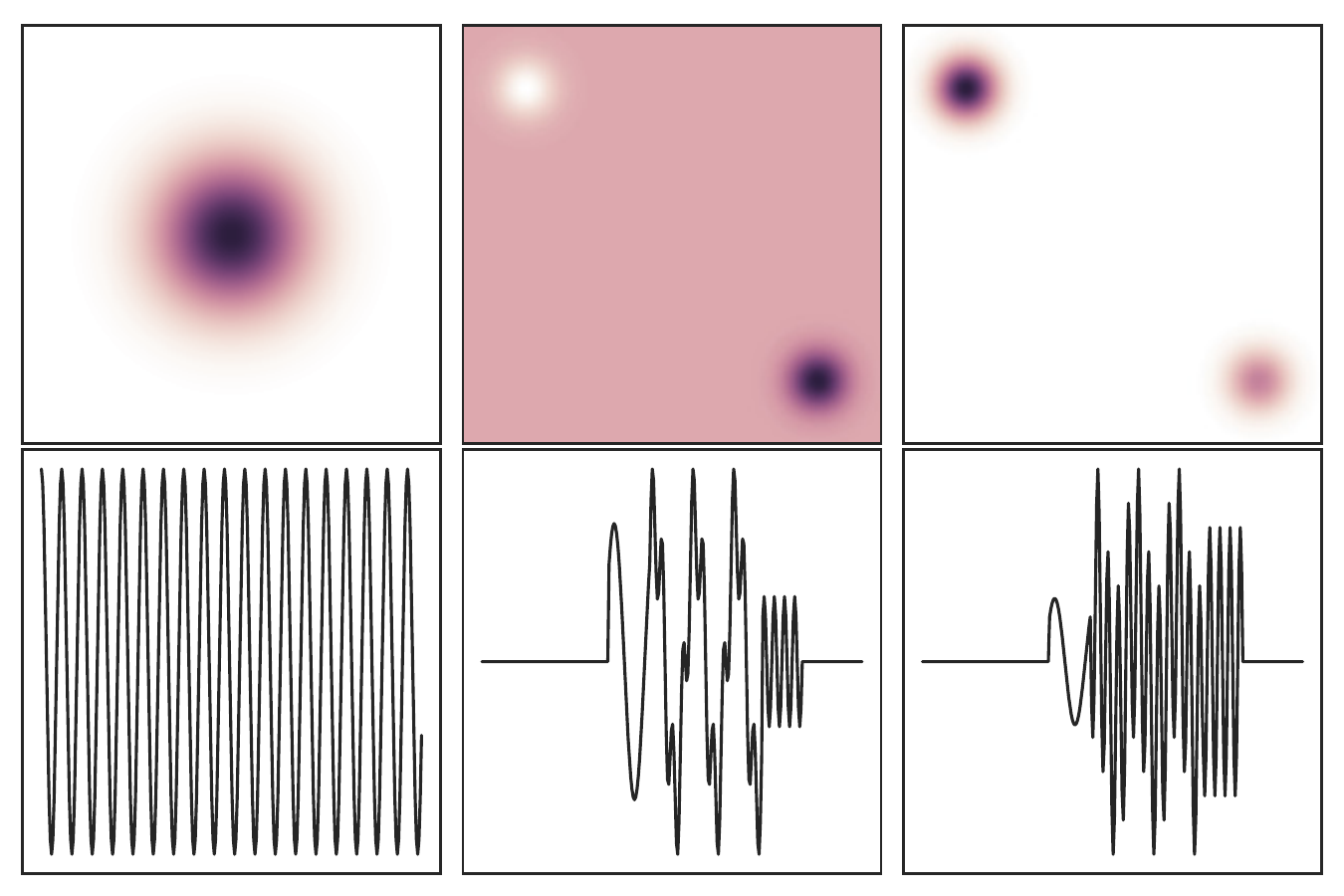}
		\caption{PCA. }
		\label{fig:pca_modes_vid}
	\end{subfigure}	
	
	\caption{Multiscale video reconstruction. SPCA successfully decomposes the video into the true dynamics, while PCA fails to disambiguate modes 2 and 3. }
	\label{fig:ms_video}
\end{figure}

First, we consider a case where spatiotemporal dynamics are generated from three spatial modes oscillating at different frequencies in overlapping time intervals:
\begin{equation}
	\mathbf{x}(t) = \sum_{j=1}^3 a_k(t) \boldsymbol{\phi}_j.
\end{equation}
The multiscale time dynamics switch on and off irregularly, i.e., the modes effectively appear mixed in time as is common in other real-world phenomena such as weather, climate, etc., illustrated in Figure~\ref{fig:multi_model}.
Consequently, within a single frame, the three modes occasionally mix, rendering the disambiguation task more challenging. However, we see that SPCA is able to recover the three modes in an unsupervised manner.  
Specifically, the data is generated on a $200\times 200$ spatial grid for 150 seconds with timestep $\Delta t=.5$s.
We flatten the spatial dimensions to obtain a data matrix with $p=40,000$ measurements for each of the $n=300$ snapshots (observations in time). 
%

The results of PCA and SPCA on the raw frame data are compared and contrasted in Fig.~\ref{fig:ms_video}. Here the spatial coherent structures extracted by SPCA recover the generating spatiotemporal modes, while PCA is unable to do so. By seeking a parsimonious representation, SPCA is able to accurately associate spatial structures with their individual time histories. Because PCA has no such constraint, the different spatial structures remain mixed.

\paragraph{Robust SPCA}
In many applications data exhibit grossly corrupted entries that typically arise from process or measurement noise. The least-squares loss function is sensitive to outliers. Thus, SPCA tends to be biased and the results can be misleading. To overcome this, our proposed robust SPCA algorithm can be used. The Huber loss function separates the input data into a low-rank component plus a sparse component. This is demonstrated in Fig.~\ref{fig:noisy}. The robust implementation clearly separates the polluted data into a low-rank component, while capturing the additive salt and pepper noise. However, the robust implementation is computationally more demanding than the standard SPCA algorithm.

\begin{figure}[!t]
	\centering
	\begin{subfigure}[t]{0.18\textwidth}
		\centering
		\DeclareGraphicsExtensions{.pdf}
		\includegraphics[width=1\textwidth]{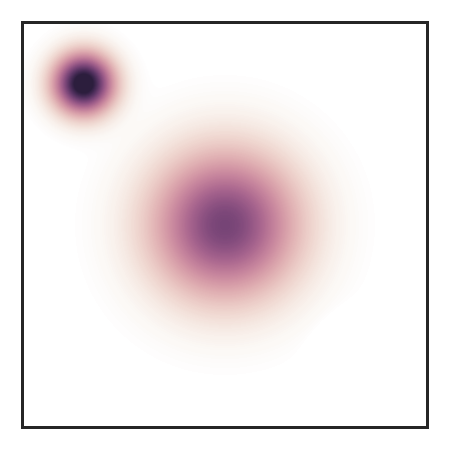}
		\caption{Truth. }
		\label{fig:noisy_true}
	\end{subfigure}
	~
	\begin{subfigure}[t]{0.18\textwidth}
		\centering
		\DeclareGraphicsExtensions{.pdf}
		\includegraphics[width=1\textwidth]{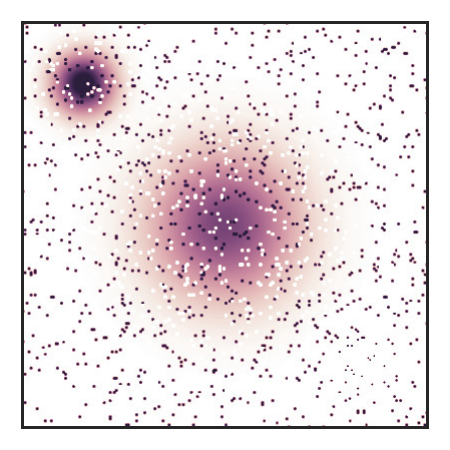}
		\caption{Corrupted. }
		\label{fig:noisy_corrupt}
	\end{subfigure}
	~
	\begin{subfigure}[t]{0.18\textwidth}
		\centering
		\DeclareGraphicsExtensions{.pdf}
		\includegraphics[width=1\textwidth]{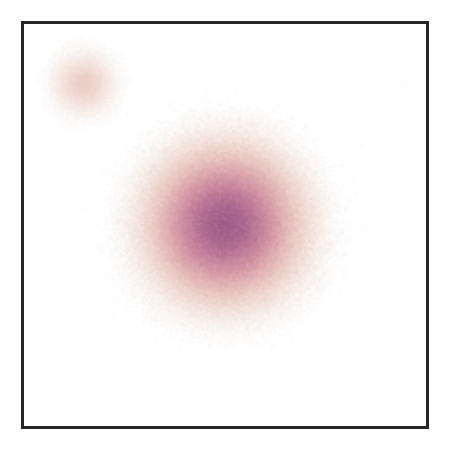}
		\caption{Reconstruction. }
		\label{fig:noisy_approx}
	\end{subfigure}
	~	
	\begin{subfigure}[t]{0.18\textwidth}
		\centering
		\DeclareGraphicsExtensions{.pdf}
		\includegraphics[width=1\textwidth]{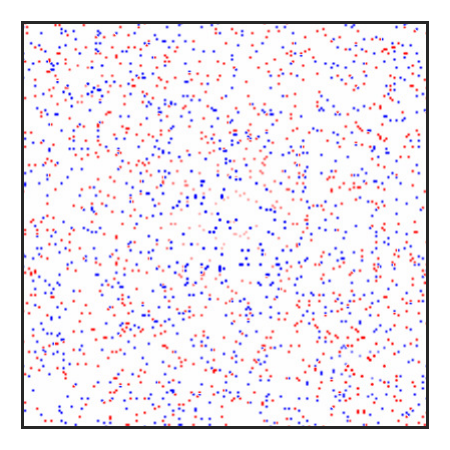}
		\caption{Outliers. }
		\label{fig:noisy_sparse}
	\end{subfigure}	
	\caption{Approximation of a grossly corrupted multiscale video using robust SPCA. Here the low-rank approximation with robust PCA (c) successfully recovers the true frame and filters out added salt and pepper noise (d). However, it can be seen that there is a slight shrinkage effect in the reconstructed frame which is introduced by the robust SPCA algorithm.}
	\label{fig:noisy}
\end{figure}

\subsection{Fluid Flow Example}

PCA has been extensively used in fluid dynamics for decades, where it is known as proper orthogonal decomposition (POD), providing a data-driven generalization of the Fourier transform~\cite{berkooz1993proper}.
Here we apply SPCA to the flow behind a cylinder, a canonical example in fluid dynamics~\cite{Noack2003jfm}. The data consists of a time series of the vorticity field behind a solid cylinder at Reynolds number 100, which induces laminar vortex shedding downstream. 
The flow is simulated using an immersed boundary projection method~\cite{taira2007immersed} on a $450\times 200$ spatial grid for three dimensionless time units with timestep $\Delta t=.02$. 
Again, we flatten the spatial dimensions and obtain a data matrix with $p=90,000$ measurements for each of the $n=150$ snapshots (observations in time). 
The resulting principal components or spatial modes of the flow  are widely used for reduced-order modeling, prediction, and control. 

The SPCA and PCA eigenmodes are compared in Fig.~\ref{fig:fluid}. Both decompositions successfully identify the dominant mode pairs that occur at characteristic harmonic frequencies. 
However, the mode structures extracted by SPCA are well-bounded and more interpretable, resulting in visible weakening downstream and stronger influence upstream. This is typical of the vortex shedding regime as vortices dissipate while advecting downstream and is not observed in the PCA modes.

Standard PCA has beneficial orthonormality properties that are crucial for projection based reduced-order modeling of high-dimensional systems.
However, as experiments and models simulate increasingly complex flows, the field is rapidly moving towards more interpretable decompositions for learning and control. 
Recent directions in network analysis of turbulence and mixing  require robust tracking of sparse spatial structures and vortices.
The ability of SPCA to delineate boundaries of vortex dynamics are critical for the scalable decomposition of such high-resolution flow data. 
Furthermore, SPCA is purely data-driven and works equally well for modal decomposition of high-fidelity computational fluid dynamics (CFD) simulation, as well as robust denoising of experimental data generated by particle image velocimetry and other high-resolution imaging techniques.

\begin{figure}[!t]
	\centering
	\begin{subfigure}[t]{0.45\textwidth}
		\centering
		\DeclareGraphicsExtensions{.pdf}
		\includegraphics[width=1\textwidth]{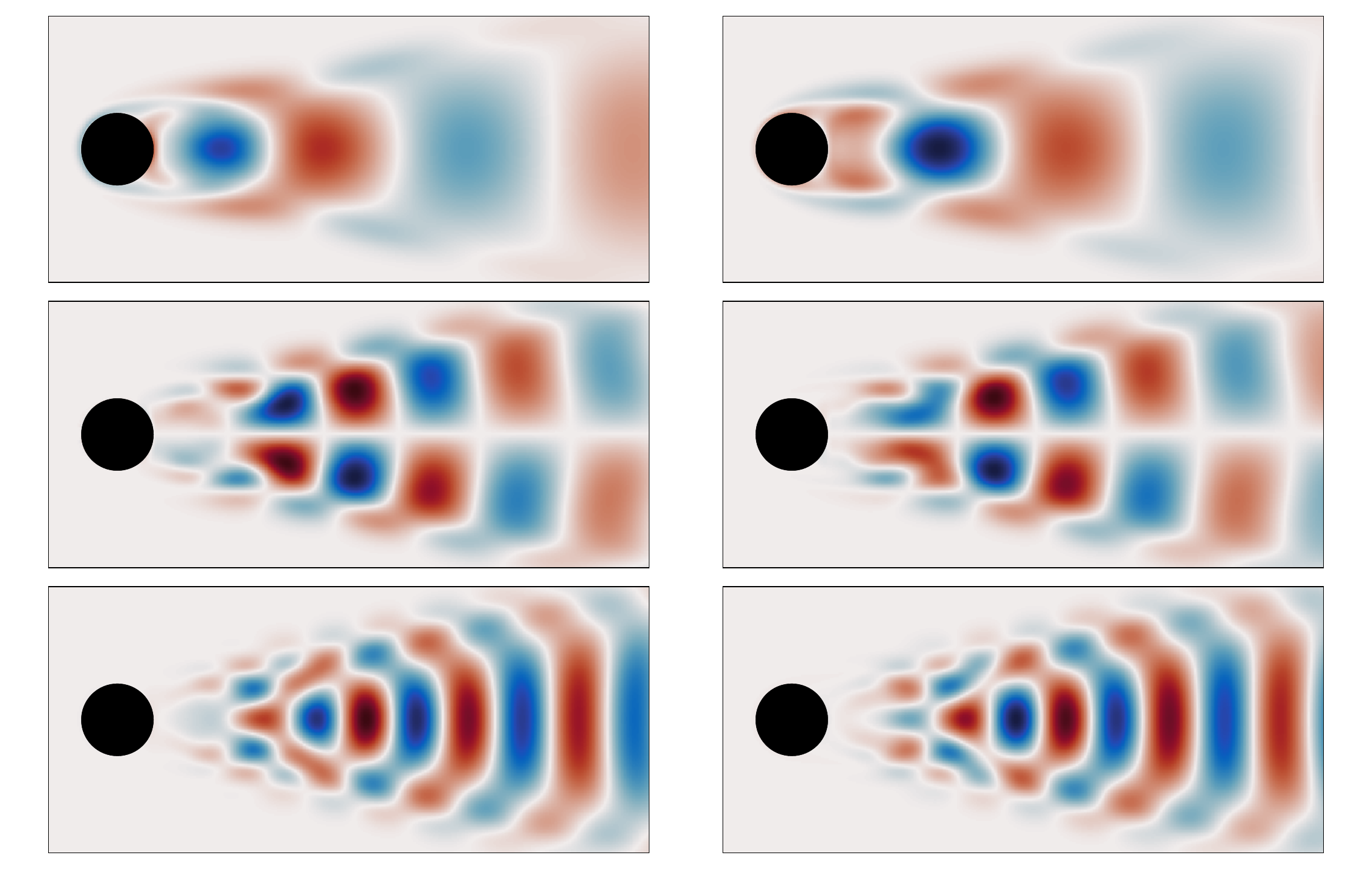}
		\caption{PCA. }
		\label{fig:fluid_pca}
	\end{subfigure}
	~
	\begin{subfigure}[t]{0.45\textwidth}
		\centering
		\DeclareGraphicsExtensions{.pdf}
		\includegraphics[width=1\textwidth]{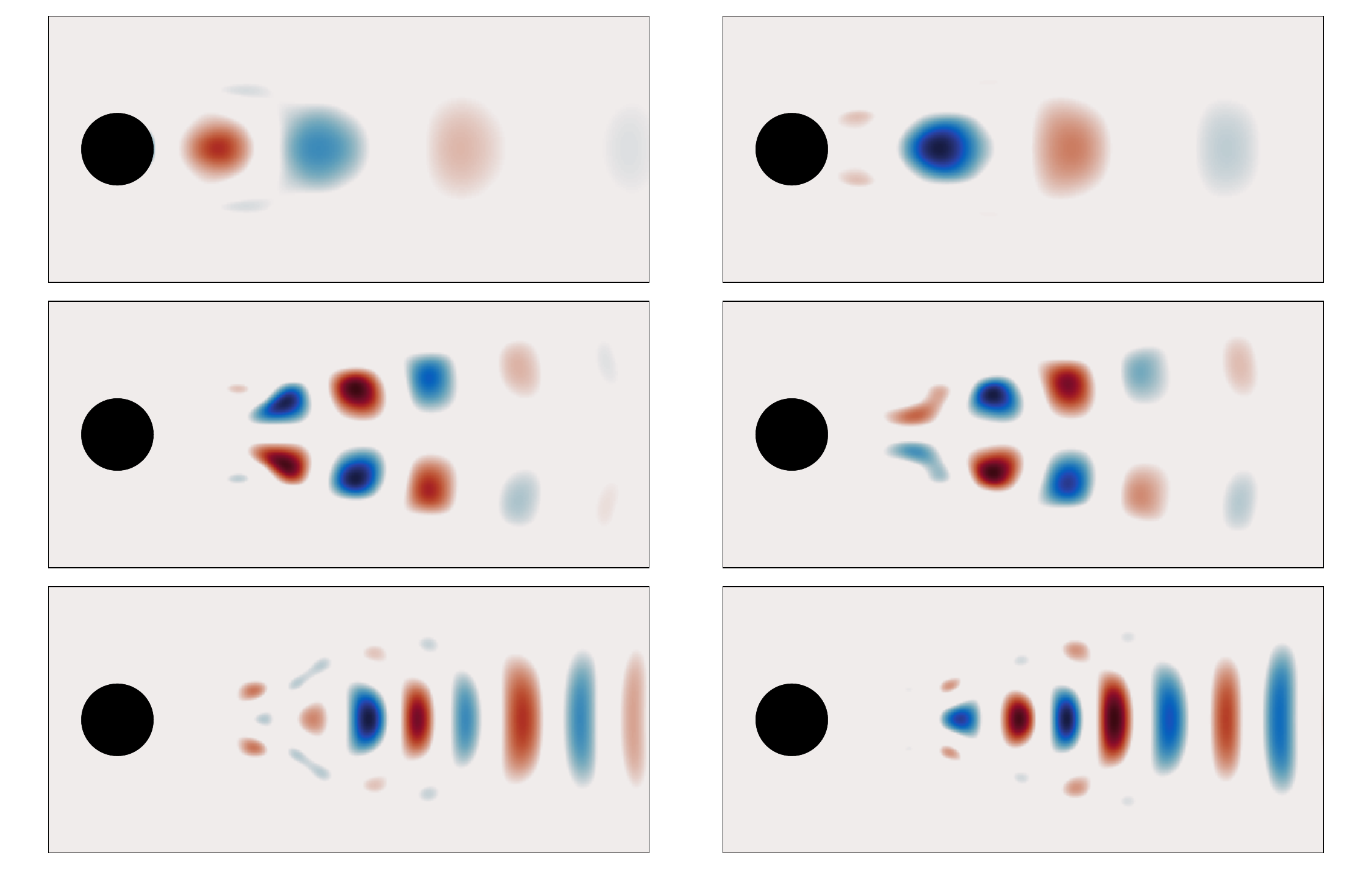}
		\caption{SPCA with $\ell_1$ regularization ($\alpha=10^{-5}$). }
		\label{fig:fluid_spca}
	\end{subfigure}
	
	\begin{subfigure}[t]{0.45\textwidth}
		\centering
		\DeclareGraphicsExtensions{.pdf}
		\includegraphics[width=1\textwidth]{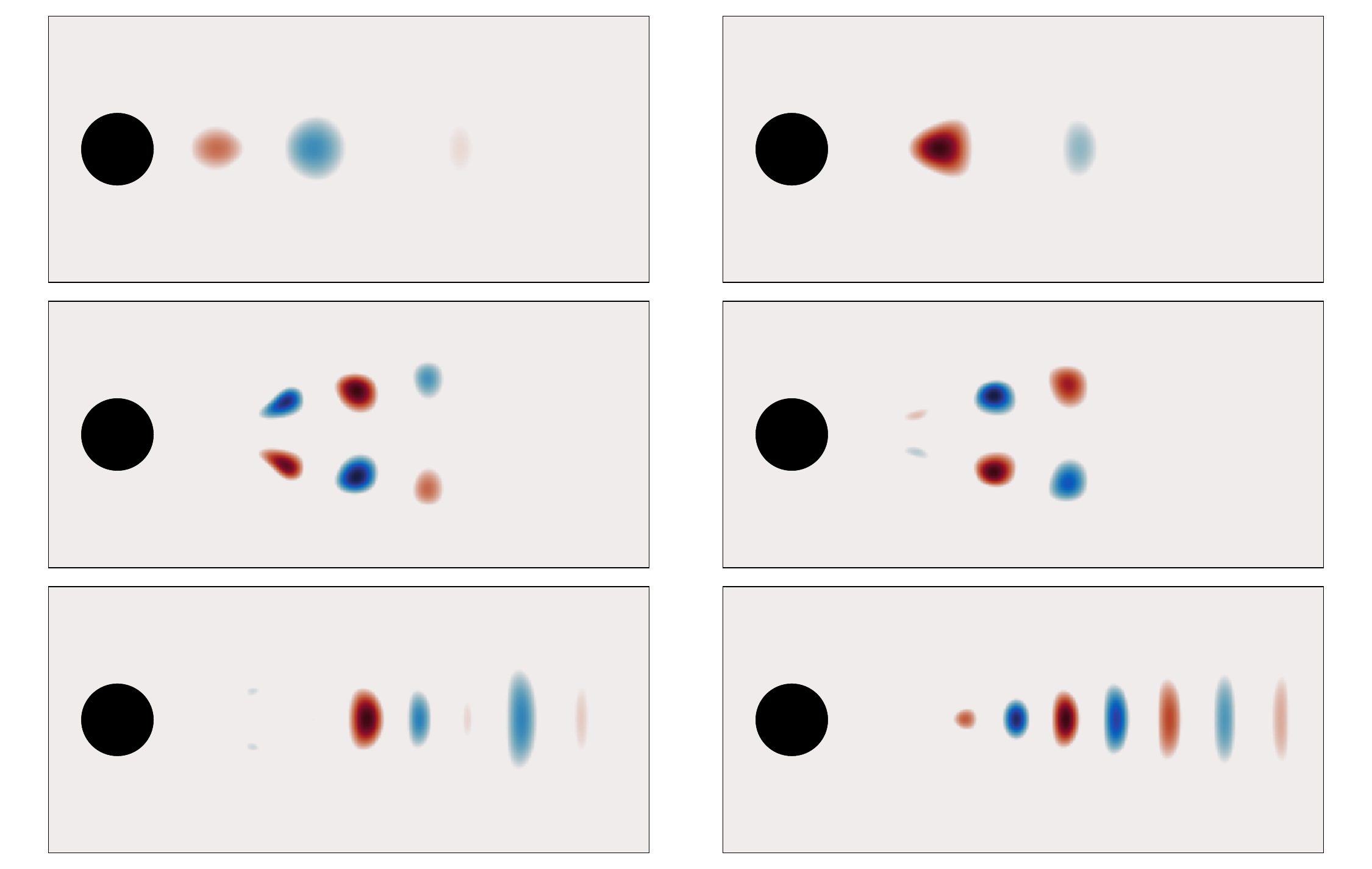}
		\caption{SPCA with $\ell_1$ regularization ($\alpha=10^{-4}$). }
		\label{fig:fluid_spca2}
	\end{subfigure}	
	~
	\begin{subfigure}[t]{0.45\textwidth}
		\centering
		\DeclareGraphicsExtensions{.pdf}
		\includegraphics[width=1\textwidth]{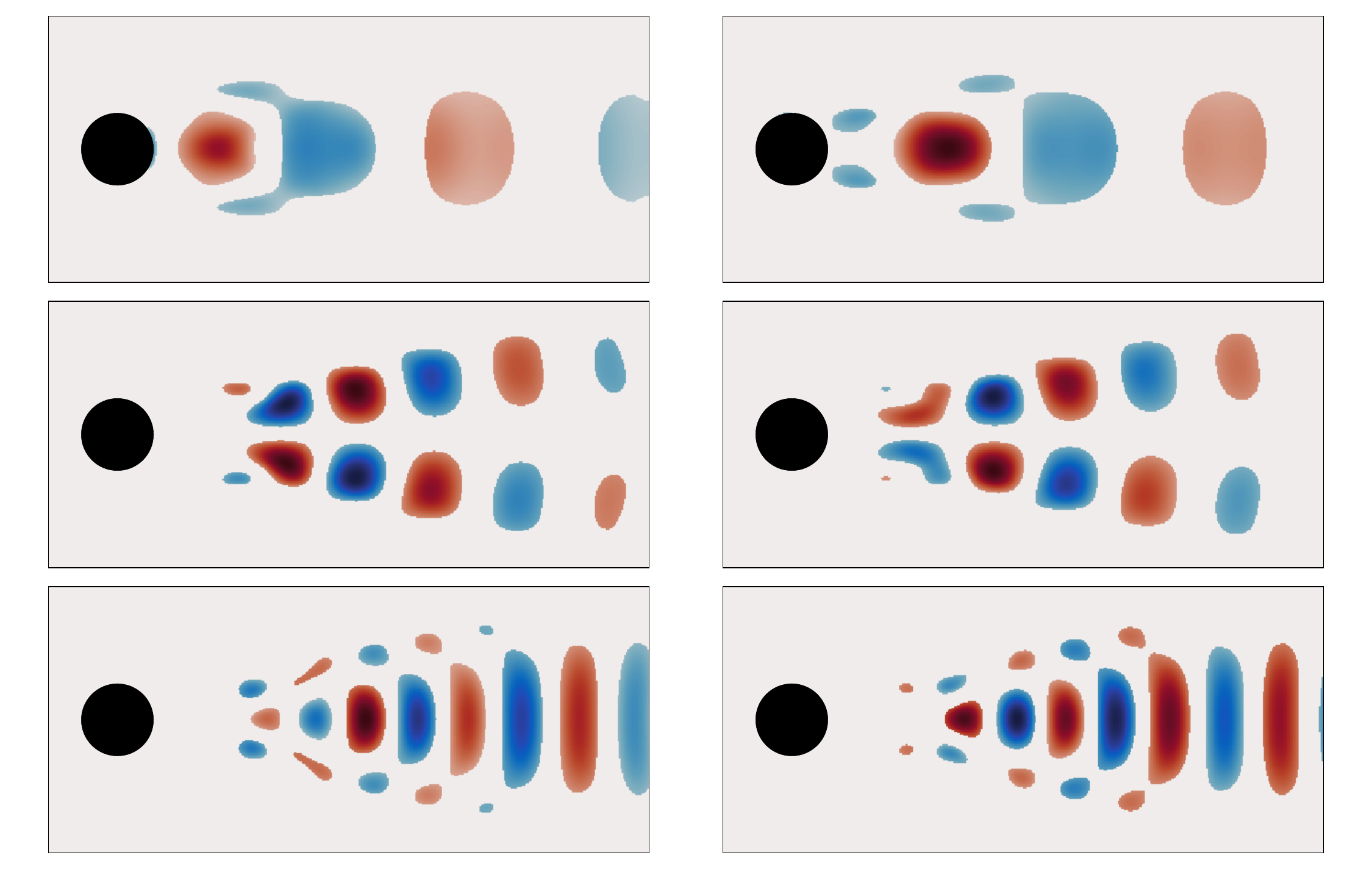}
		\caption{SPCA with $\ell_0$ regularization ($\alpha=10^{-5}$). }
		\label{fig:fluid_spca3}
	\end{subfigure}				
	
	\caption{Sparse PCA demonstrates superior separation of the spatial modes responsible for vortex shedding. As a result, we can better differentiate their spatial influence on different regions of the flow downstream of the cylinder. }
	\label{fig:fluid}
\end{figure}

\subsection{Sea Surface Temperature Example}

We now apply SPCA to satellite ocean temperature data from 1990-2017~\cite{reynolds2002improved}, and compare SPCA results to PCA.\footnote{The data are provided by the NOAA and are accessible via their Web site at \url{https://www.esrl.noaa.gov/psd/}.} 
The data consists of $n=1,458$ temporal snapshots which measure the weekly temperature means at $360\times 180 = 64,800$ spatial grid points. Since we omit data over continents and land, the ambient dimension reduces to $p=44,219$ observations in our analysis. 
Our objective is the accurate identification of the intermittent El Ni\~no and La Ni\~na warming events, which are famously implicated in global weather patterns and climate change. 
The El Ni\~no Southern Oscillation (ENSO) is defined as any sustained temperature anomaly above running mean temperature with a duration of 9 to 24 months. 
In climate sciences, principal components are also known as empirical orthogonal functions or EOFs; however, traditional PCA struggles to find a low-rank representation of this complex, high-dimensional system. 

The canonical El Ni\~no is associated with a narrow band of warm water off coastal Peru that is commonly referred as NI\~NO 1+2, 3, 3.4,  or 4 to differentiate the types of bands. 
Traditional PCA is unable to isolate this band, instead combining it with broader spatial signatures across the Pacific and Atlantic in mode 4 (Fig.~\ref{fig:pca_modes_sst}). 
Nevertheless, this mode is often used to compute the canonical Oceanic Ni\~no Index (ONI). 
On the other hand, SPCA obtains a dramatic and clean separation of NI\~NO 1-4 within the 4th mode (Fig.~\ref{fig:spca_modes_sst}). 
This is contextualized by the associated temporal mode, which yields sharper peaks during the 1997-1999 and 2014-2016 major El Ni\~no events compared to PCA. 
The 12-month moving average of the temporal modes for both PCA and SPCA is shown in Fig.~\ref{fig:sst_env} and confirms that SPCA differentiates major and minor ENSO events with greater clarity than PCA. Indeed, SPCA clearly isolates the fourth ENSO mode as the last physically relevant component to the system.

Previous study of this dataset has required a multiresolution time-frequency separation of the data matrix in order to clearly identify the ENSO mode in an unsupervised manner~\cite{Kutz2016mrdmd}. 
Without the sparsity constraint, SVD-based methods struggle to obtain a low-rank representation of these complex systems with nonlinear dynamics, coupled interactions and multiple timescales of motion. 
SPCA has the potential to yield sparse modal representations of complex systems and coherent structures that may alter our understanding of oceanic and atmospheric phenomena.

\subsection{Denoising with Sparse PCA}

Cumulative variance plots reveal that sparse PCA behaves differently on the latter two examples. This can be attributed to the level of stochasticity in each system. 
The cylinder data has high temporal resolution and is therefore sufficiently well-resolved for sparse PCA to capture nearly all the variance within the low-rank component (Fig.~\ref{fig:fluid_cvar}). In this case the decomposition is similar to PCA, although spatially more localized. On the other hand, the ocean data has coarse weekly temporal resolution. Therefore, faster dynamics which are not sufficiently resolved appear stochastic. Hence, slower timescales (annual, ENSO) are reflected in the low-rank component of SPCA as indicated by  cumulative variance (Fig.~\ref{fig:sst_cvar}). PCA, however, overfits with `noisy' components which are not physically meaningful.

\begin{figure}[!t]
	\centering
	\begin{subfigure}[t]{0.4\textwidth}
		\centering
		\DeclareGraphicsExtensions{.png}
		\includegraphics[width=1\textwidth]{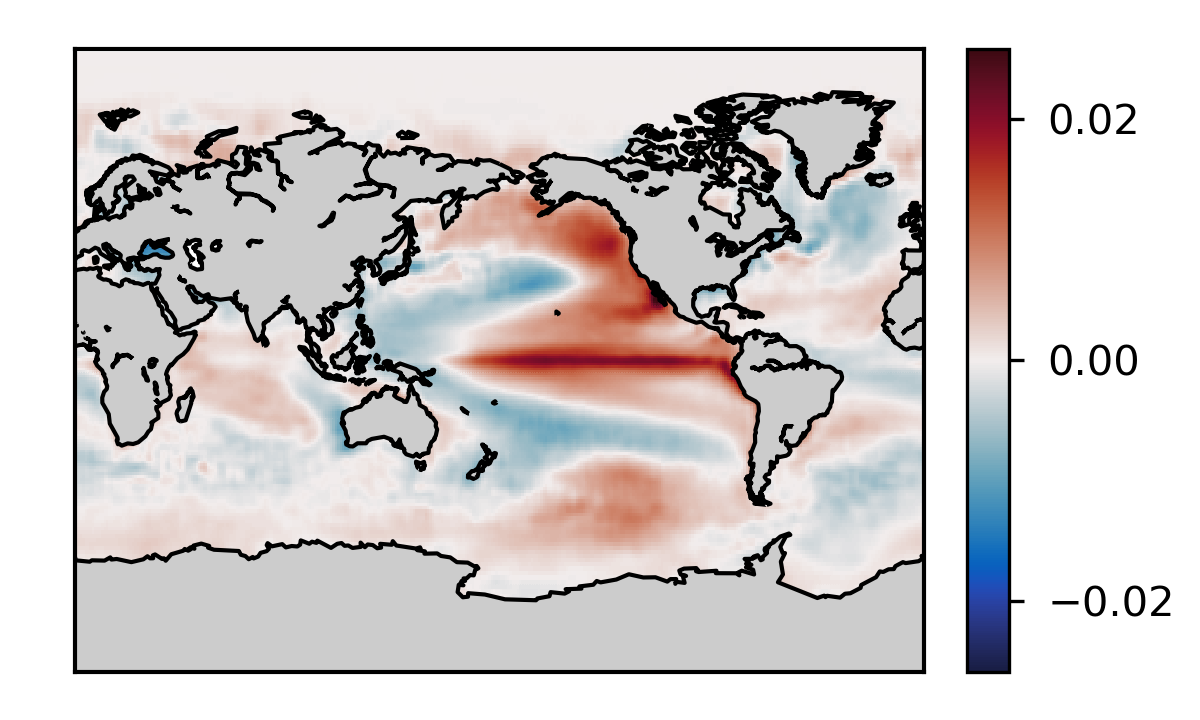}
		\caption{PCA. }\vspace{0.2cm}
		\label{fig:pca_modes_sst}
	\end{subfigure}
	~
	\begin{subfigure}[t]{0.4\textwidth}
		\centering
		\DeclareGraphicsExtensions{.png}
		\includegraphics[width=1\textwidth]{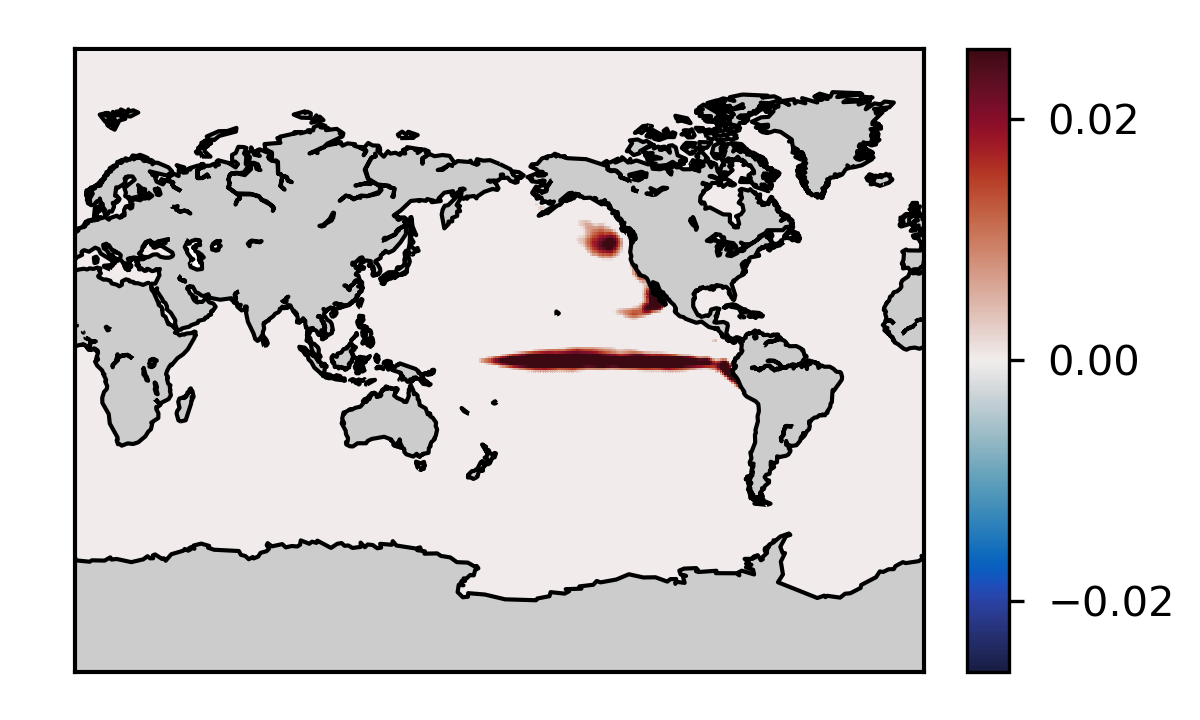}
		\caption{SPCA with $\ell_1$ regularization ($\alpha=10^{-4}$). }\vspace{0.2cm}
		\label{fig:spca_modes_sst}
	\end{subfigure}
	
	\vspace{-.10in}		
	\begin{subfigure}[t]{0.4\textwidth}
		\centering
		\DeclareGraphicsExtensions{.png}
		\includegraphics[width=1\textwidth]{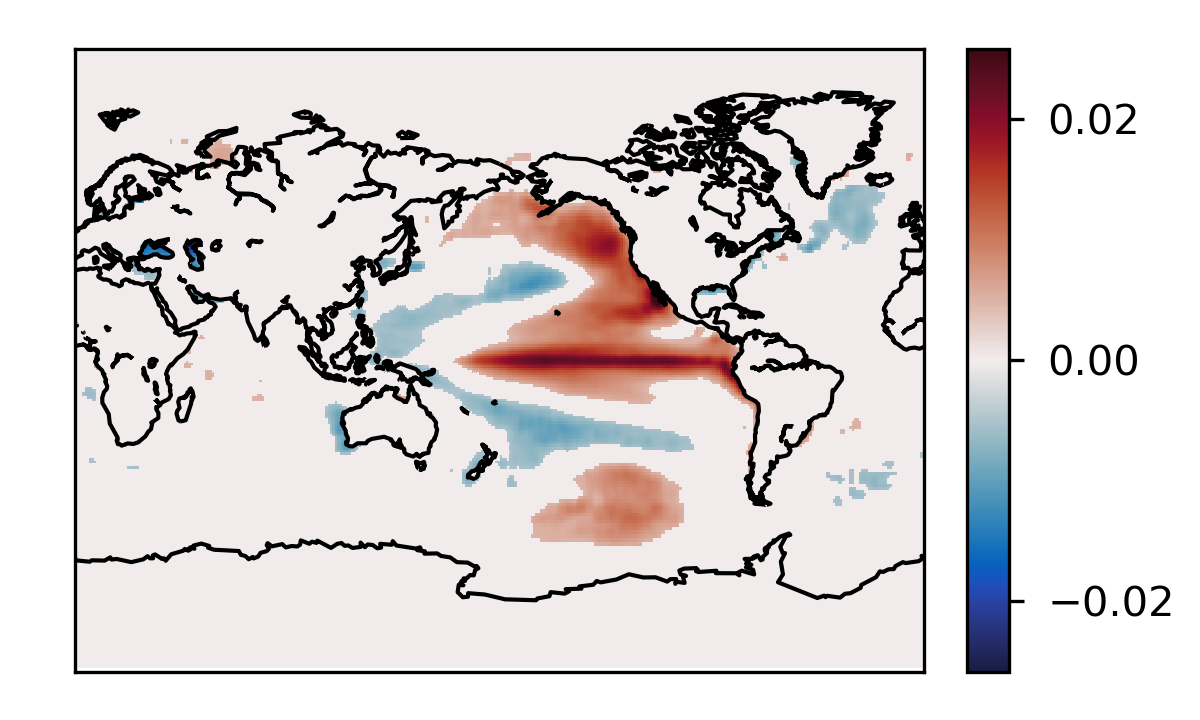}
		\caption{SPCA with $\ell_0$ regularization ($\alpha=10^{-5}$). }
		\label{fig:spca_modes_sst_l0}
	\end{subfigure}
	~
	\begin{subfigure}[t]{0.4\textwidth}
		\centering
		\DeclareGraphicsExtensions{.png}
		\includegraphics[width=1\textwidth]{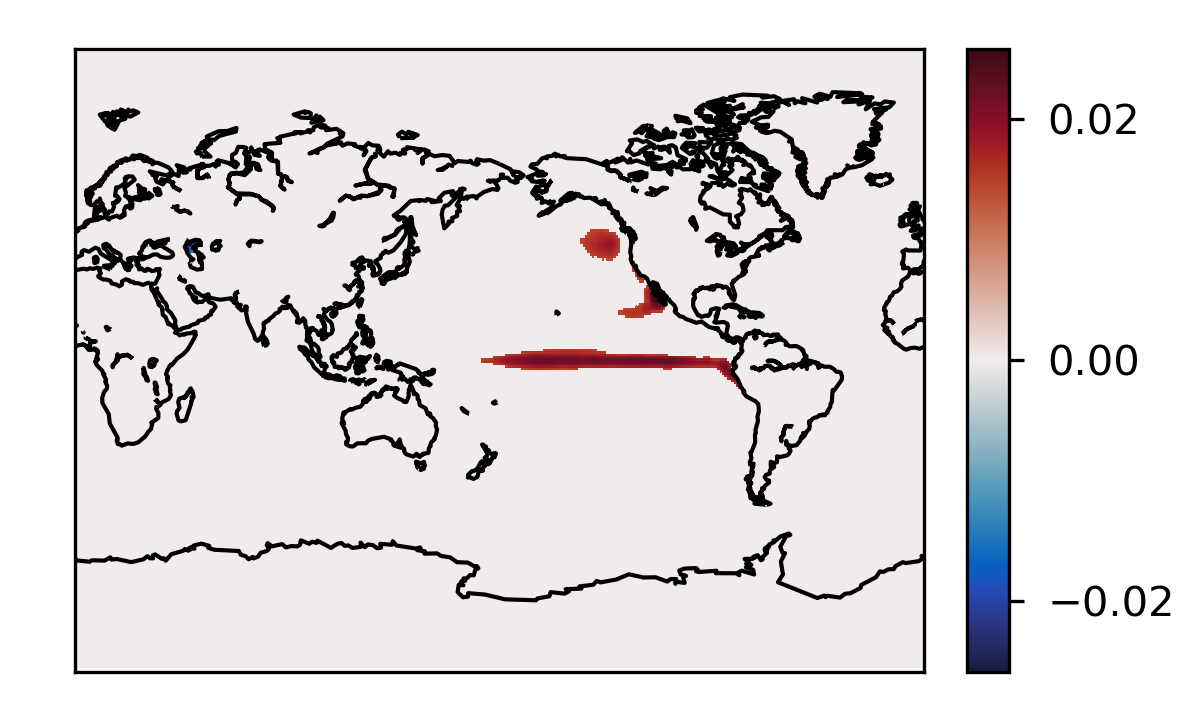}
		\caption{SPCA with $\ell_0$ regularization ($\alpha=10^{-4}$). }
		\label{fig:spca_modes_sst_l0_2}
	\end{subfigure}			
	
	\caption{SPCA successfully identifies the band of warmer temperatures in the South Pacific traditionally associated with El Ni\~no. By contrast, the corresponding PCA mode picks up spurious spatial correlations across the globe.}
	\label{fig:sst}
\end{figure}

\begin{figure}[!t]
	\centering
	\begin{subfigure}[t]{0.4\textwidth}
		\centering
		\DeclareGraphicsExtensions{.pdf}
		\includegraphics[width=1\textwidth]{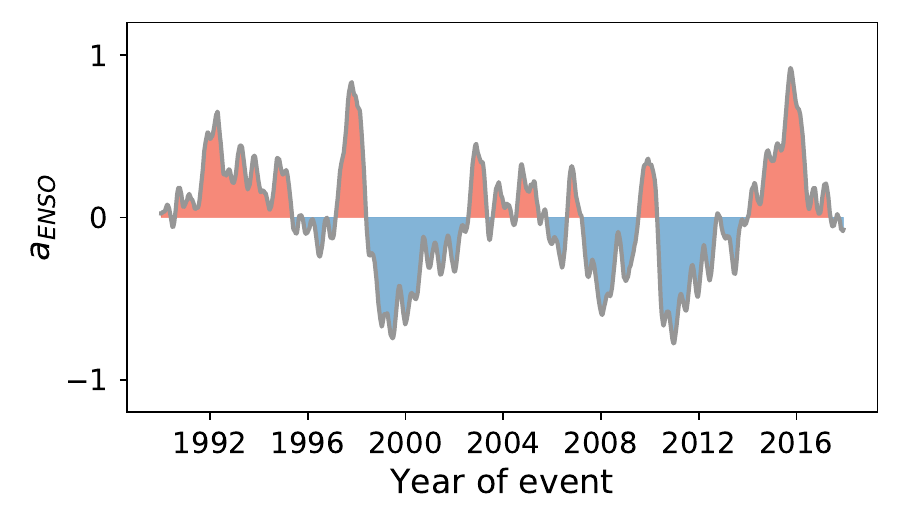}
		\caption{PCA modes. }
	\end{subfigure}
	~
	\begin{subfigure}[t]{0.4\textwidth}\hspace{-1cm}
		\centering
		\DeclareGraphicsExtensions{.pdf}
		\includegraphics[width=1\textwidth]{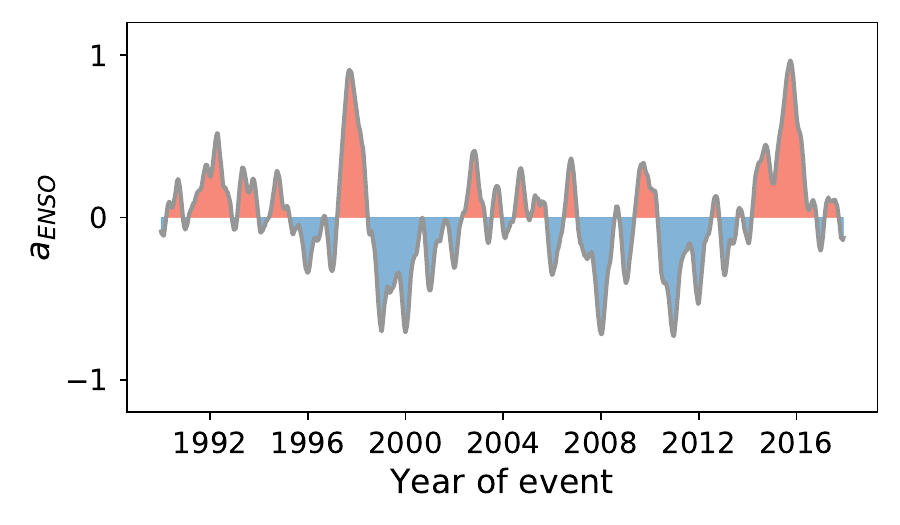}
		\caption{SPCA with $\ell_1$ regularization ($\alpha=10^{-4}$). }
	\end{subfigure}
	
	\caption{ Oceanic Ni\~no Index (ONI), a 12-month moving average of the ENSO mode, reveals greater distinction between major (1997-1999,2014-2016) and minor events with SPCA modes.}
	\label{fig:sst_env}
\end{figure}

\begin{figure}[H]
	\centering
	
	\begin{subfigure}[t]{0.4\textwidth}
		\centering
		\DeclareGraphicsExtensions{.pdf}
		\includegraphics[width=1\textwidth]{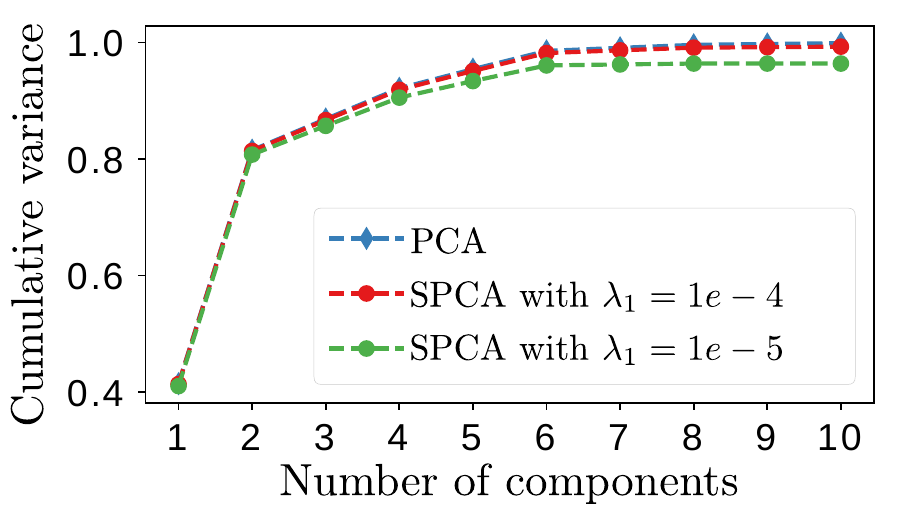}
		\caption{Flow behind a cylinder. }
		\label{fig:fluid_cvar}
	\end{subfigure}
	~
	\begin{subfigure}[t]{0.4\textwidth}
		\centering
		\DeclareGraphicsExtensions{.pdf}
		\includegraphics[width=1\textwidth]{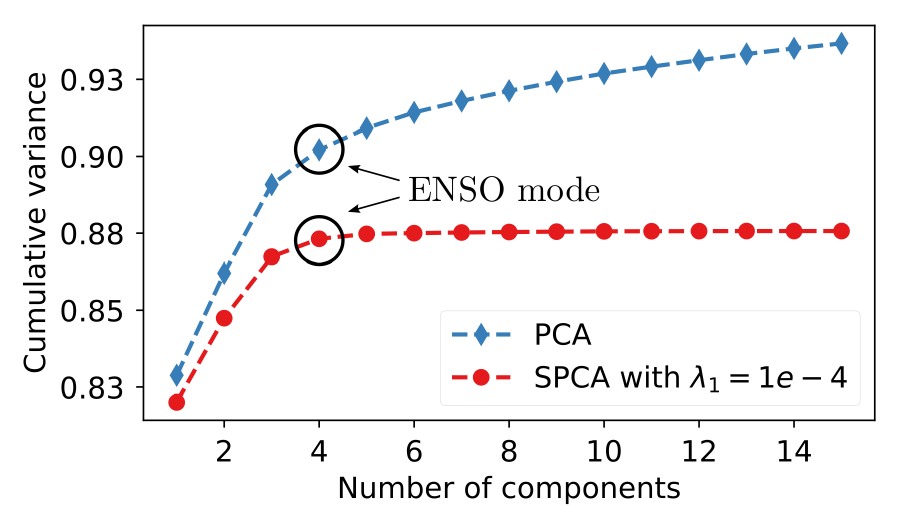}
		\caption{Sea surface temperature data.}
		\label{fig:sst_cvar}
	\end{subfigure}
	
	\caption{Cumulative variance of each component. Although sparsity promotes spatially localized structure, SPCA retains nearly all of the variance of the fluid flow behind a cylinder (a).  In contrast to PCA, SPCA separates the ENSO mode from noisy contributions even though ENSO captures only 1\% of the total variance (b).} 
	\label{fig:sst_cvar_overview}
\end{figure}

\subsection{Computational Performance}

\begin{figure}[!b]
	\centering
	\begin{subfigure}[t]{0.48\textwidth}
		\centering
		\DeclareGraphicsExtensions{.pdf}
		\includegraphics[width=1\textwidth]{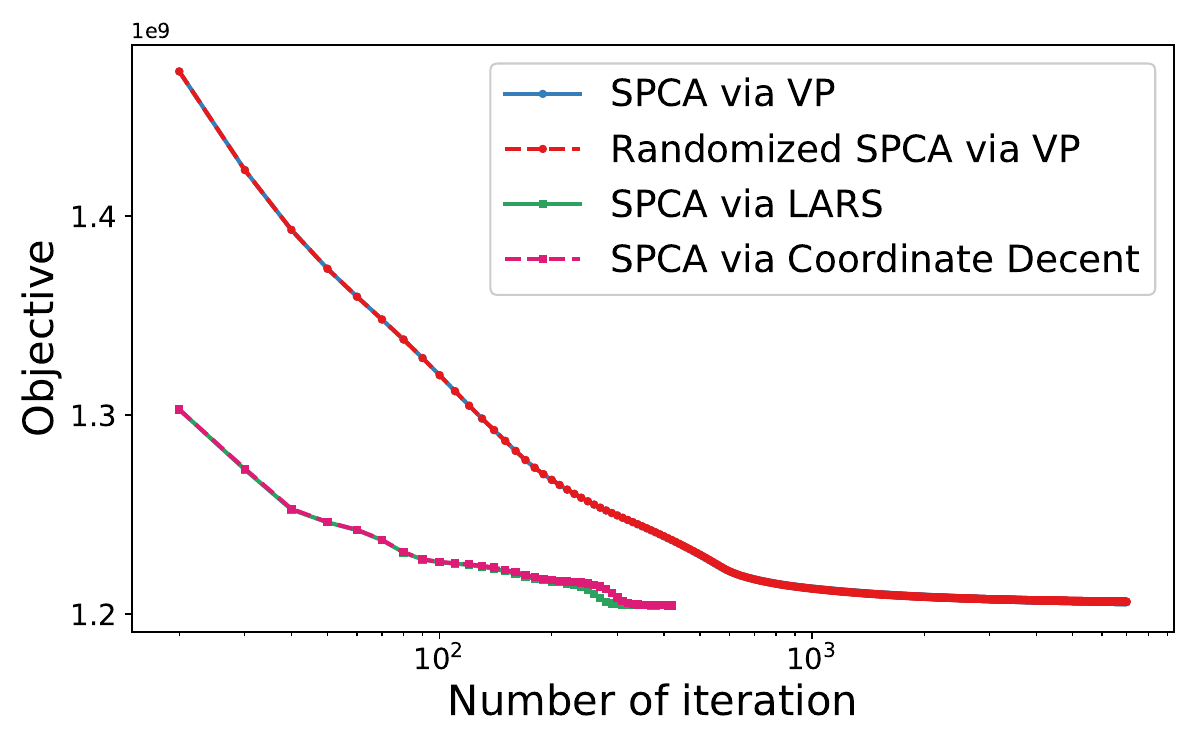}
		\caption{Objective vs number of iterations.}
	\end{subfigure}
	~
	\begin{subfigure}[t]{0.48\textwidth}
		\centering
		\DeclareGraphicsExtensions{.pdf}
		\includegraphics[width=1\textwidth]{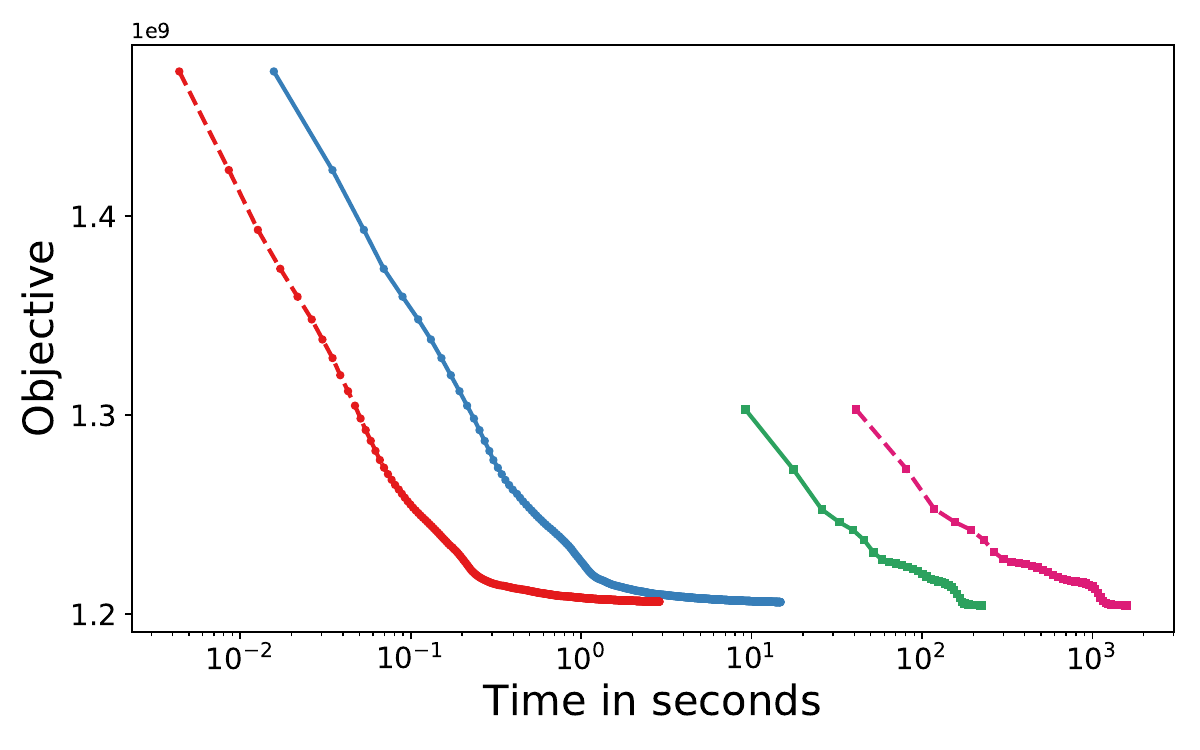}
		\caption{Objective vs computational time.}
	\end{subfigure}
	
	\caption{Computational performance of different SPCA algorithms. The  dominant $10$ sparse weight vectors are computed for a $2000\times 1344$ data matrix.}
	\label{fig:comp_performace}
\end{figure}

\begin{figure}[!b]
	
	\begin{subfigure}[t]{0.48\textwidth}
		\centering
		\DeclareGraphicsExtensions{.pdf}
		\includegraphics[width=1\textwidth]{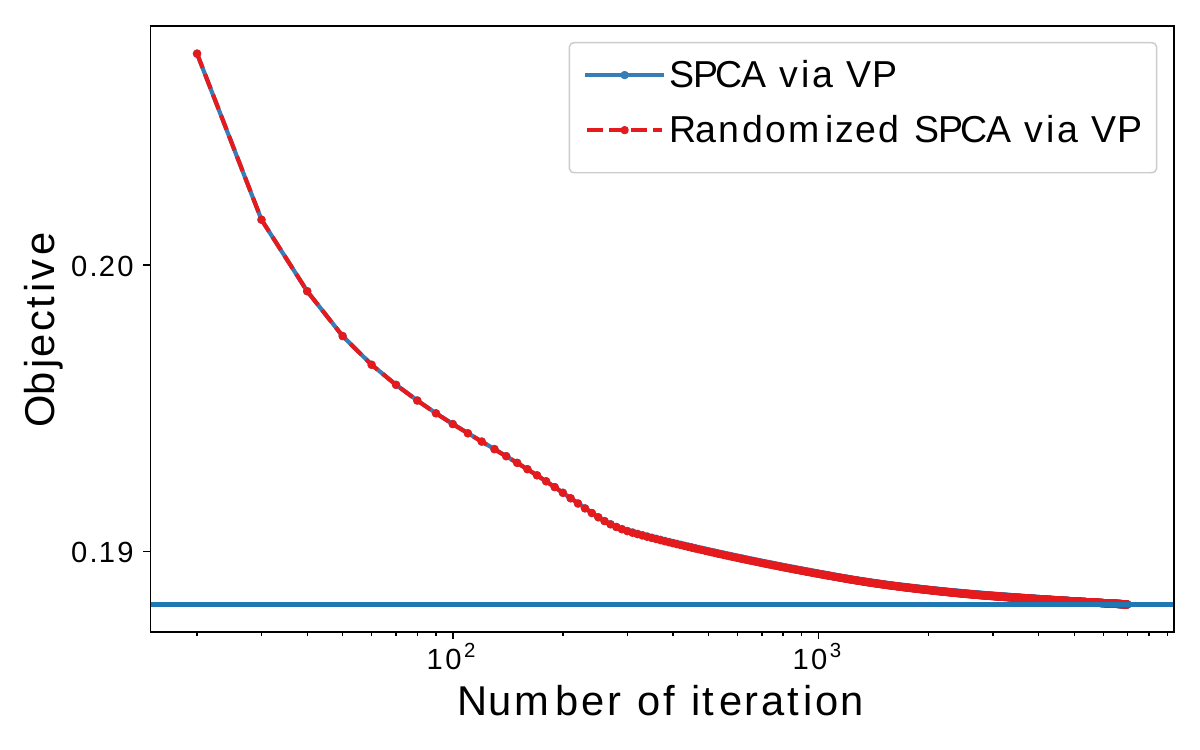}
		\caption{Objective vs number of iterations.}
	\end{subfigure}
	~
	\begin{subfigure}[t]{0.48\textwidth}
		\centering
		\DeclareGraphicsExtensions{.pdf}
		\includegraphics[width=1\textwidth]{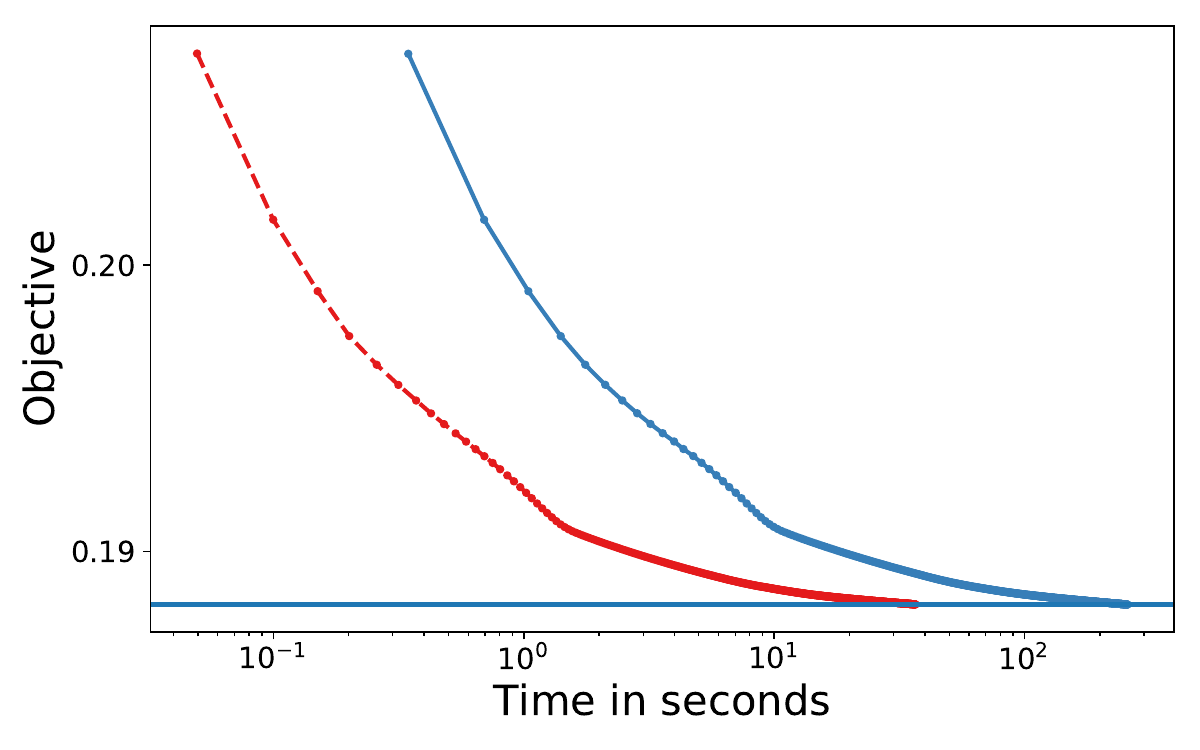}
		\caption{Objective vs computational time.}
	\end{subfigure}
	
	\caption{Computational performance of the randomized and deterministic SPCA algorithm using variable projection. The dominant $10$ sparse weight vectors are computed for a $2000\times 16128$ data matrix. The randomized algorithms is about $4$ times faster.}
	\label{fig:comp_performace2}
\end{figure}

To demonstrate the computational performance of the proposed SPCA algorithms we compute the leading $k=10$ components for two data matrices. First, we consider the cases of small $p$ data. Figure~\ref{fig:comp_performace} shows the number of iterations and time until the objective function converges within a tolerance level of $10^{-5}$. 
For comparison we show the performance of the SPCA algorithm using least angle regression (LARS) and coordinate descent (CD) as proposed by~\cite{zou2006sparse}. Our proposed algorithm based on variable projection outperforms both the LARS and CD algorithms in terms of the computational time. This is despite the fact that our algorithm requires more iterations. Clearly, the per iteration costs of the variable projection algorithm are substantially less than the computational costs of the LARS and CD algorithm.  

Further, the randomized accelerated SPCA algorithm outperforms the deterministic variable projection algorithms. The desired accuracy is achieved about $5$ times faster compared to the deterministic algorithm. This is despite the fact that the randomized algorithms require more iterations than the deterministic algorithm to converge.
The computational advantage is even greater for the high-dimensional data setting (i.e., big $p$) as shown in Figure~\ref{fig:comp_performace2}   
The computational advantage of the randomized algorithm becomes pronounced with increasing dimensions of the input matrix. Hence, the randomized algorithm allows exploring a large space of tuning parameters and is well suited for performing cross-validation. 

Implementations of our algorithms are provided in Python \url{https://github.com/erichson/ristretto} and in R \url{https://CRAN.R-project.org/package=sparsepca}.

\section{Discussion}\label{sec:discussion}

We have presented a robust and scalable architecture for computing sparse principal component analysis (SPCA). Specifically, we have modeled SPCA as a matrix factorization problem with orthogonality constraints, and developed specialized optimization algorithms that partially minimize a subset of the variables (variable projection).  Our SPCA algorithm is scalable and robust, greatly improving computational efficiency over current state-of-the-art methods  while retaining comparable performance. 
More precisely, we have demonstrated that:
(i) The value function view approach provides an efficient and flexible framework for SPCA;
(ii) Robust SPCA can be formulated using the Huber loss;
(iii) A wide variety of sparsity-inducing regularizers can be incorporated into the framework;
(iv) The proposed algorithms are computationally efficient for high-dimensional data, i.e, large $p$;
(v) Randomized methods for linear algebra substantially eases the computational demands, while obtaining a near-optimal approximation for low-rank data. 

SPCA is a useful diagnostic tool for data featuring rich dynamics that give rise to multiscale structures in both space and time.  Given that such phenomena are ubiquitous in the physical, engineering, biological, and social sciences, this work provides a valuable tool for improved interpretability, especially in the diagnostics of localized structures and disambiguation of distinct time scale physical processes.
The work also opens a number of avenues for future development:

\paragraph{Methodological Extensions}  
This scalable approach for identifying spatially localized spatial structures in high-dimensional and multiscale data may be directly applied to (1) tensor decompositions~\cite{comon2014tensors,bro1997parafac,Erichson2017tensor}, which represent data in a multi-dimensional array structure, (2) parsimonious dynamical systems models~\cite{brunton2016pnas}, which identify the fewest nonlinear interactions required to capture the underlying physical mechanisms, and (3) \emph{in situ} sensing and control, where sensors and actuators are generally required to be spatially localized~\cite{Manohar2017csm}. 

\paragraph{Applications in the Engineering and Physical Sciences.} 
The methods developed here will be broadly applicable to dynamical systems that are high-dimensional, multiscale, and where there is a need for interpretable and parsimonious models for prediction, estimation, and control.  
Specific applications where SPCA has already been applied include atmospheric chemistry~\cite{gmd-12-1525-2019}, genomics~\cite{abraham2014fast,lee2010super}, and biological systems~\cite{ma2011principal} more broadly.  
In addition, there is tremendous opportunity for advances in diverse fields, such as improving climate prediction, detecting and controlling structures in the brain, and closed loop control of turbulent fluid systems~\cite{Brunton2015amr}.

\section*{Acknowledgments}
The authors would like to thank Daniela Witten for helpful discussions in the early stage of writing this manuscript. We would also like to express our gratitude to the two anoymous revierews for their insightful comments and pointing out the connection between PALM and our proposed algorithm. 
NBE would like to acknowledge the generous fundingsupport from the Defense Advanced Research Projects Agency (DARPA) and the Air ForceResearch Laboratory (FA8750-17-2-0122) as well as Amazon Web Services for supporting theproject with EC2 credits.  
Research of AYA was partially supported by the Washington Research Foundation Data Science Professorship.  
SLB gratefully acknowledges funding support from the Army Research Office grant W911NF-17-1-0306. 
JNK acknoledges support from the Air Force Office of Scientific Research grant FA9550-17-1-0329.

\appendix		

\section{Overview of unstructured and structured sparsity promoting regularizers}\label{app:regularizers}

\subsection{Unstructured Sparsity}

\begin{figure}[!b]
	\centering
	
	\begin{subfigure}[t]{0.23\textwidth}
		\begin{center}
			\scalebox{0.85}{
				\begin{tikzpicture}[auto,node distance = 2cm,>=latex']
					
					\coordinate (X1) at (-2,0);
					\coordinate (X2) at (2,0);
					\coordinate (Y1) at (0,-2);
					\coordinate (Y2) at (0,2);
					
					\draw [-latex, black, line width=1.0pt]  (X1) -- (X2);
					\draw [-latex, black, line width=1.0pt]  (Y1) -- (Y2);
					
					\draw [-, darkred, line width=1.0pt, dashed]  (0,0) -- (0,1);
					\draw [-, darkred, line width=1.0pt, dashed]  (0,0) -- (-1,0);
					\draw [-, darkred, line width=1.0pt, dashed]  (0,0) -- (0,-1);
					\draw [-, darkred, line width=1.0pt, dashed]  (0,0) -- (1,0);
					
					\draw [-, darkred, line width=1.0pt]  (0,0) -- (0,0.5);
					\draw [-, darkred, line width=1.0pt]  (0,0) -- (-0.5,0);
					\draw [-, darkred, line width=1.0pt]  (0,0) -- (0,-0.5);
					\draw [-, darkred, line width=1.0pt]  (0,0) -- (0.5,0);

					\draw [-, darkblue, line width=1.0pt]  (-1.5,1.8) -- (1.5,0.3);
					\draw [fill=black] (0,1.02) circle (3pt);

					
			\end{tikzpicture}}
		\end{center}
		\caption{$\ell_0$ norm. }
	\end{subfigure}    
	\begin{subfigure}[t]{0.23\textwidth}
		\begin{center}
			\scalebox{0.85}{
				\begin{tikzpicture}[auto,node distance = 2cm,>=latex']
					
					\coordinate (X1) at (-2,0);
					\coordinate (X2) at (2,0);
					\coordinate (Y1) at (0,-2);
					\coordinate (Y2) at (0,2);
					
					\draw [-latex, black, line width=1.0pt]  (X1) -- (X2);
					\draw [-latex, black, line width=1.0pt]  (Y1) -- (Y2);
					
					\draw [-, darkred, line width=1.0pt, dashed]  (1,0) -- (0,1);
					\draw [-, darkred, line width=1.0pt, dashed]  (0,1) -- (-1,0);
					\draw [-, darkred, line width=1.0pt, dashed]  (-1,0) -- (0,-1);
					\draw [-, darkred, line width=1.0pt, dashed]  (0,-1) -- (1,0);
					
					\draw [-, darkred, line width=1.0pt]  (0.5,0) -- (0,0.5);
					\draw [-, darkred, line width=1.0pt]  (0,0.5) -- (-0.5,0);
					\draw [-, darkred, line width=1.0pt]  (-0.5,0) -- (0,-0.5);
					\draw [-, darkred, line width=1.0pt]  (0,-0.5) -- (0.5,0);
					
					\draw [-, darkblue, line width=1.0pt]  (-1.5,1.8) -- (1.5,0.3);
					\draw [fill=black] (0,1.02) circle (3pt);


			\end{tikzpicture}}
		\end{center}
		\caption{$\ell_1$ norm. }
	\end{subfigure}    
	\begin{subfigure}[t]{0.23\textwidth}
		\begin{center}
			\scalebox{0.85}{
				\begin{tikzpicture}[auto,node distance = 2cm,>=latex']
					
					\coordinate (X1) at (-2,0);
					\coordinate (X2) at (2,0);
					\coordinate (Y1) at (0,-2);
					\coordinate (Y2) at (0,2);
					
					\draw [-latex, black, line width=1.0pt]  (X1) -- (X2);
					\draw [-latex, black, line width=1.0pt]  (Y1) -- (Y2);
					
					\draw[darkred,thick,dashed] (0,0) circle (0.91cm);

					\draw[darkred,thick,] (0,0) circle (0.5cm);

					\draw [-, darkblue, line width=1.0pt]  (-1.5,1.8) -- (1.5,0.3);
					\draw [fill=black] (0.5,0.8) circle (3pt);


			\end{tikzpicture}}
		\end{center}
		\caption{$\ell_2$ norm. }
	\end{subfigure}   
	\begin{subfigure}[t]{0.23\textwidth}
		\begin{center}
			\scalebox{0.85}{
				\begin{tikzpicture}[auto,node distance = 2cm,>=latex']
					
					\coordinate (X1) at (-2,0);
					\coordinate (X2) at (2,0);
					\coordinate (Y1) at (0,-2);
					\coordinate (Y2) at (0,2);
					
					\draw [-latex, black, line width=1.0pt]  (X1) -- (X2);
					\draw [-latex, black, line width=1.0pt]  (Y1) -- (Y2);
					
					\draw [-, darkred, line width=1.0pt, dashed, bend right]  (1,0) edge (0,1);
					\draw [-, darkred, line width=1.0pt, dashed, bend right]  (0,1) edge (-1,0);
					\draw [-, darkred, line width=1.0pt, dashed, bend right]  (-1,0) edge (0,-1);
					\draw [-, darkred, line width=1.0pt, dashed, bend right]  (0,-1) edge (1,0);
					
					\draw [-, darkred, line width=1.0pt, bend right]  (0.5,0) edge (0,0.5);
					\draw [-, darkred, line width=1.0pt, bend right]  (0,0.5) edge (-0.5,0);
					\draw [-, darkred, line width=1.0pt, bend right]  (-0.5,0) edge (0,-0.5);
					\draw [-, darkred, line width=1.0pt, bend right]  (0,-0.5) edge (0.5,0);
					
					\draw [-, darkblue, line width=1.0pt]  (-1.5,1.8) -- (1.5,0.3);
					\draw [fill=black] (0,1.02) circle (3pt);

					
			\end{tikzpicture}}
		\end{center}
		\caption{Elastic net.}
	\end{subfigure}    
	
	\caption{Illustration of some norms which are used as regularizers. $\ell_0$, $\ell_1$ and elastic net are sparsity-inducing.}
	\label{Fig:geo_regularizers}
\end{figure}

The $\ell_0$ `norm', denoted $\ell_0(\bm x)$ or $\|\bm x\|_0$, counts the number of non-zero elements in a vector $\mathbf{x}$. When used as a regularizer $\psi$,
it encourages models with small cardinality, i.e., a small number of active loadings.
Although $\ell_0$ is non-smooth and non-convex, its proximal operator is simply hard thresholding (see Table~\ref{table:reg}).
%

In many applications, the $\ell_1$ norm is used to approximate $\ell_0$. 
%
In the context of least squares problems, using $\ell_1$  is known as LASSO (least absolute shrinkage and selection operator). The proximal operator of the scaled $\ell_1$ norm
\(
\gamma \|\bm x \|_{1}
\)
is the {\it soft-thresholding} operator, see Table~\ref{table:reg}.

One drawback of the $\ell_1$ norm is that it tends to activate only one coefficient from any set of highly correlated variables. 
The elastic net, introduced by Zou and Hastie~\cite{zou2005regularization},  overcomes 
this drawback, using a linear combination of the $\ell_1$ and quadratic penalties:
\begin{equation*}
	\psi_{12}(\bm x) = \alpha\|\bm x \|_{1} + \beta\|\bm x \|_{2}^2.
\end{equation*}
The elastic net has an implicit grouping effect that is particularly useful for the analysis of high-dimensional multiscale physical systems, where we want to find all the associated variables which correspond to an underlying mode, 
rather than selecting only one variable from each underlying mode.  
The proximal operator of $\psi_{12}$ combines scaling and soft thresholding, see Table~\ref{table:reg}.
Following the same idea, we can also combine $\ell_0$ and the quadratic penalty:
\begin{equation*}
	\psi_{02}(\bm x) = \alpha\|\bm x\|_0 + \beta \|\bm x\|_2^2.
\end{equation*}
The $\psi_{02}$ regularizer detects correlated sets of very sparse predictors, 
and its proximal operator of $\psi_{02}$ combines scaling and hard thresholding, see Table~\ref{table:reg}.
Figure~\ref{Fig:geo_regularizers} illustrates these regularizers.  
Many other examples of proximal operators are collected in~\cite{combettes2011proximal}.

\begin{table}[!tb]
	\caption{Regularizers $\psi$ and their proximal operators.\label{table:reg}}
	\centering
	\scalebox{0.95}{
		\begin{tabular}{l  c   c }
			\toprule 
			Symbol & Regularizer $\psi$ & $\mathrm{prox}_{\gamma \psi}(\bm x)_i$ \\ \midrule 
			$\psi_0$ & $\|\bm x\|_0$ & 
			$
			\mathrm{Hard\; Thresholding}: \quad
			\begin{cases}
				x_i, & x_i^2 > 2\gamma\\
				0, & \text{otherwise}
			\end{cases} 
			$ \\ 
			\midrule 
			$\psi_1$ & $\|\bm x\|_1$ &
			$
			\mathrm{Soft\; Thresholding}: \quad
			\begin{cases}
				x_i - \gamma, & x_i > \gamma \\
				x_i + \gamma, & x_i < -\gamma \\
				0, & \text{otherwise}
			\end{cases}
			$ \\ 
			\midrule 
			$\psi_{02}$ & $\alpha\|\bm x\|_0 + \beta\|\bm x\|_2^2$ &
			$
			\mathrm{Scaled\; Hard\; Thresholding}: \quad
			\begin{cases}
				x_i/(1+2\gamma\beta), &x_i^2 > 2\gamma\alpha(1+2\gamma\beta)\\
				0, & \text{otherwise}
			\end{cases}
			$ \\ 
			\midrule 
			$\psi_{12}$ & $\alpha\|\bm x\|_1 + \beta\|\bm x\|_2^2$ &
			$
			\mathrm{Scaled\; Soft\; Thresholding}: \quad
			\begin{cases}
				(x_i - \gamma\alpha)/(1+2\gamma\beta), & x > \gamma\alpha\\
				(x_i + \gamma\alpha)/(1+2\gamma\beta), & x < -\gamma\alpha\\
				0, &\text{otherwise}
			\end{cases}
			$ \\
			\bottomrule
	\end{tabular}}
\end{table}

\subsection{Structured Sparsity}

A large number of separable structured regularizers $\psi$ can be used in the proposed SPCA framework.
Separability ensures that the prox-operator can be computed either in closed form or using a routine for both convex and nonconvex regularizers.  
Here we highlight two examples.

In some applications, selection occurs between groups of variables known {\it a priori}. 
The group lasso regularizer~\cite{yuan2006model} enforces that all the variables corresponding 
to these predefined groups are either activated or set to $0.$
Its prox operator can be written as
\[
\mathrm{prox}_{\gamma \|\cdot\|_2}(\bm x) = \begin{cases}
	(1 - \gamma/\|\bm x\|_2) \bm x, & \|\bm x\|_2 > \gamma\\
	\bm 0, & \|\bm x\|_2 \le \gamma
\end{cases}.
\]
An extension is the sparse group lasso~\cite{simon2013sparse}, which adds an additional $\ell_1$ penalty for each group.
Another useful regularizer is the fused lasso~\cite{tibshirani2005sparsity}, which gives a way to incorporate 
information about spatial or temporal structure in the data.

\section{The Orthogonal Procrustes Problem}\label{sec:Procrustes}

We seek an orthonormal matrix $\mathbf{A}$ so that
\begin{equation}
	\mathbf{A} = \argmin_{\mbf A} \,\, \fnorm{\mbf X - \mbf X \mbf B \mbf A^\top}^2 \quad \text{s.t.} \quad \mathbf{A}^\top \mathbf{A}=\mathbf{I}.
\end{equation}
Indeed, a closed form solution is provided by the SVD.
First,we expand the above objective function as
\begin{align*}
	\argmin_{\mbf A}  \,\, \|\mathbf{X}\|^2_F + \|\mathbf{XB}\|^2_F - 2 \cdot \Tr(\mathbf{X}^\top \mathbf{XBA^\top}).
\end{align*}
This problem is equivalent to finding a orthonormal matrix $\mathbf{A}$ which maximizes $\Tr(\mathbf{X}^\top \mathbf{XBA^\top})$. We proceed by substituting the SVD of $\mbf X^\top \mbf X \mbf B$ and obtain
\begin{equation}\label{eq:protrmax}
	\argmax_{\mbf A} \,\, \Tr(\mathbf{U \Sigma V^\top A^\top}) = \Tr(\mathbf{\Sigma V^\top   A^\top U}).
\end{equation}
Note that $\mathbf{\Sigma}$ is a diagonal matrix with non-negative entries and $\mathbf{V^\top A^\top U}$ is an orthonormal matrix for any orthonormal matrix  $\mathbf{A^\top}$. Because of this, the trace norm in Eq.~\eqref{eq:protrmax} is maximized by the value of $\mathbf{A^\top}$ that turns $\mathbf{V^\top A^\top U}$ into an identity matrix $\mathbf{I}$, in order to yield $\Tr(\mathbf{\Sigma I})$. Hence, an optimal solution is provided by $\mathbf{A=UV^\top }$, i.e., the left and right singular vectors of  $\mbf X^\top \mbf X \mbf B$.   

\section{Proof of Theorem}\label{sec:proof}

\subsection{Technical Preliminaries}
\label{sec:notation}
In the following we give a brief overview of notation and concepts used to develop and analyze the algorithms in this paper. Further, we review briefly the elements of variational analysis for the theoretical analysis of the algorithm~\cite{mord1,RW98}.
\subsubsection{Matrix Spaces}
We consider the collection of all matrices with the same dimension $\R^d$ (where $d$ could be shorthand for $p\times p$) as a Hilbert space equipped with the inner product.
More concretely, the inner product is defined by the trace
and the norm induced by this inner product is the Frobenius norm 
\begin{equation*}
	\ip{\mathbf{M}}{\mathbf M} := \Tr(\mathbf{M}^\top\mathbf{M}) = \fnorm{\mathbf M}^2.
\end{equation*}
For any map ${\bm \Phi}: \R^d \to \R^l$, we set,
\begin{equation*}
	\Lip({\bm \Phi}) := \sup_{\mbf M \ne \mbf N} \frac{\fnorm{{\bm \Phi}(\mbf M) - {\bm \Phi}(\mbf N)}}{\fnorm{\mbf M - \mbf N}}
\end{equation*}
We say that ${\bm \Phi}$ is $L$-Lipschitz continuous, for some $L\ge0$, if the inequality $\Lip({\bm \Phi})\le L$ holds.
\subsubsection{Functions and Geometry} Constraints, such as those in~\eqref{eq:spca_obj}, can be represented
using functions from a matrix space $\mathbb{R}^d$ to the extended real line defined by $\overline{\R}:=\R \cup \{\pm \infty\}$.
The {\it domain} and the {\it epigraph} of any function $f:\R^d\to\R$ are the defined sets
\begin{align*}
	\textrm{dom}\, f&:=\{\mbf M\in \R^d: f(\mbf M)<+\infty\},\\
	\textrm{epi}\, f&:=\{(\mbf M,r)\in \R^d\times \R: f(\mbf M)\leq r\}.
\end{align*}
For any set $\mathcal{F}\subset\R^d$, we define the {\em distance}, {\em projection} and {\em indicator functions} for $\mbf M\in\R^d$ by
\begin{align*}
	\dist(\mbf M;\mathcal{F}) &:= \inf_{\mbf N\in \mathcal{F}}~\|\mbf N - \mbf M\|, \quad \proj(\mbf M;\mathcal{F}) := \argmin_{\mbf N\in \mathcal{F}}~\|\mbf N - \mbf M\|,\\
	\delta_{\mathcal{F}}(\mbf M) &:= \begin{cases}
		0, & \mbf M\in \mathcal{F}\\
		\infty, & \mbf M\not\in \mathcal{F}
	\end{cases}.
\end{align*}
For $\mathbb{O}:=\{\mbf A \in \R^d: \mbf A^T \mbf A = \mbf I\}$ in~\eqref{eq:spca_obj}, and given $\mbf M = \mbf{U\Sigma V^T}$, we have 
\begin{equation}
	\label{eq:ortho}
	\dist(\mbf M;\mathbb{O}) := \|\mbf I - \mbf \Sigma\|^2, \quad \proj(\mbf M;\mathbb{O}) := \mbf{UV^T}, \quad 
	\delta_{\mathbb{O}}(\mbf M) := \begin{cases}
		0, & \mbf \Sigma = \mbf I\\
		\infty, & \mbf \Sigma \neq \mbf I.
	\end{cases}.
\end{equation}

\subsubsection{Subgradients and Subdifferentials} Characterizing stationarity (a necessary condition for optimality) 
is a key step in analyzing the behavior of an algorithm and deriving practical termination criteria. 
Problem~\eqref{eq:spca_obj} is nonsmooth, so gradients do not exist. Instead, we can use 
more general concepts of {\it subgradients}, which exist for nonsmooth, nonconvex functions.  

Consider an arbitrary function $f\colon\R^d \to \overline \R$ and a point $\overline{\mbf M}$ with $f(\overline{\mbf M})$ finite. 
When $f$ is convex, the subgradient of $f$ at $\overline{\mbf M}$ is defined as the collection of tangent affine minorants: 
\begin{equation}
	\label{eq:subgradient}
	\partial f(\overline{\mbf M}) := \{\mbf V: f(\mbf M)\geq f(\overline{\mbf M})+\langle \mbf V,\mbf M-\overline{\mbf M}\rangle\}.
\end{equation}
If $f$ is differentiable at $\overline{\mbf M}$, then $\partial f(\overline{\mbf M})$ contains only one element, and it is a gradient. When $f$ is not differentiable, 
the subdifferential can contain multiple elements (see Figure~\ref{fig:subgradients}). From~\eqref{eq:subgradient}, it is clear that $0\in \partial f(\overline{\mbf M})$ 
implies that $ f(\mbf M)\geq f(\overline{\mbf M})$ for all $\mbf M$, i.e. $\overline{\mbf M}$ is a global minimum. 
%

When $f$ is nonconvex,~\eqref{eq:subgradient} may not hold globally for any $\mbf V$, and we need a localized definition. 
The {\em Fr\'{e}chet subdifferential} of $f$ at $\overline{\mbf M}$, denoted $\hat \partial f(\overline{\mbf M})$,
is the set of all matrices $\mbf V$ that satisfy  
%
\[
f(\mbf M)\geq f(\overline{\mbf M})+\langle \mbf V,\mbf M-\overline{\mbf M}\rangle+o(\|\mbf M-\overline{\mbf M}\|)\
\]
%
as $\mbf M \to \overline{\mbf M}$. 
The inclusion $\mbf V\in\hat\partial f(\overline{\mbf M})$ holds precisely when the affine function $\mbf M\mapsto f(\overline{\mbf M})+\langle \mbf V,\mbf M-\overline{\mbf M}\rangle$ underestimates $f$ up to first-order near $\overline{\mbf M}$.
The limit of Fr\'{e}chet subgradients $v_i\in \hat\partial f(\mbf M_i)$ along a sequence $\mbf M_i\to\overline{\mbf M}$ may not be a Fr\'{e}chet subgradient at the limiting point $\overline{\mbf M}$.  
The {\it limiting subdifferential} $\partial f(\overline{\mbf M})$ is the set of all matrices $\mbf V$ for which there exist 
sequences $\mbf M_i$ and $\mbf V_i$ that satisfy $\mbf V_i\in \partial f(\mbf M_i)$ and $(\mbf M_i,f(\mbf M_i),\mbf V_i)\to (\overline{\mbf M}, f(\overline{\mbf M}),\mbf V)$.
In the nonconvex case, the stationarity condition $0 \in \partial f(\overline{\mbf M})$ no longer implies global (or local) optimality. However, it is still a necessary condition, 
and one that can be checked. We characterize stationarity of~\eqref{eq:spca_obj} by the distance of $0$ to the limiting subdifferential $\partial f(\overline{\mbf M})$.

\begin{figure}[!t]
	\centering
	\begin{subfigure}[t]{0.3\textwidth}
		\begin{tikzpicture}
			\begin{axis}[
				thick,
				axis lines = middle,
				enlargelimits = true,
				width=.99\textwidth, height=3cm,
				xmin=-2, xmax=2, ymin=0, ymax=2,
				no markers,
				samples=50,
				axis lines*=left, 
				axis lines*=middle, 
				scale only axis,
				xtick={-2, 2},
				ytick={ 0, 2},
				] 
				\addplot[line width=2.0pt, darkred, domain=-2:2]{x^2};
			\end{axis}
		\end{tikzpicture}
		\caption{$f(x) = x^2$}
	\end{subfigure}
	~
	\begin{subfigure}[t]{0.3\textwidth}
		\begin{tikzpicture}
			\begin{axis}[
				thick,
				axis lines = middle,
				enlargelimits = true,
				width=.99\textwidth, height=3cm,
				xmin=-2, xmax=2, ymin=0, ymax=2,
				no markers,
				samples=50,
				axis lines*=left, 
				axis lines*=middle, 
				scale only axis,
				xtick={-2,  0,  2},
				ytick={0, 2},
				] 
				\addplot[line width=2.0pt, darkred, domain=-2:2]{abs(x)};
			\end{axis}
		\end{tikzpicture}
		\caption{$f(x) = |x|$}
	\end{subfigure}
	~
	\begin{subfigure}[t]{0.3\textwidth}
		\begin{tikzpicture}
			\begin{axis}[
				thick,
				axis lines = middle,
				enlargelimits = true,
				width=.99\textwidth, height=3cm,
				xmin=-2, xmax=2, ymin=-2, ymax=0,
				no markers,
				samples=50,
				axis lines*=left, 
				axis lines*=middle, 
				scale only axis,
				xtick={-2,  0,  2},
				ytick={-2, 0},
				] 
				\addplot[line width=2.0pt, darkred, domain=-2:2]{-abs(x)};
			\end{axis}
		\end{tikzpicture}
		\caption{$f(x) = -|x|$}
	\end{subfigure}
	
	\vspace*{0.1cm}
	
	\begin{subfigure}[t]{0.3\textwidth}
		\begin{tikzpicture}
			\begin{axis}[
				thick,
				axis lines = middle,
				enlargelimits = true,
				width=.99\textwidth, height=3cm,
				xmin=-2, xmax=2, ymin=-2, ymax=2,
				no markers,
				samples=50,
				axis lines*=left, 
				axis lines*=middle, 
				scale only axis,
				xtick={-2, 2},
				ytick={ -2, 2},
				] 
				\addplot[line width=2.0pt, darkred, domain=-2:2]{x};
			\end{axis}
		\end{tikzpicture}
		\caption{$\partial (\cdot)^2(0) = \nabla (\cdot)^2 (0)= 0.$}
	\end{subfigure}
	~
	\begin{subfigure}[t]{0.3\textwidth}
		\begin{tikzpicture}
			\begin{axis}[
				thick,
				axis lines = middle,
				enlargelimits = true,
				width=.99\textwidth, height=3cm,
				xmin=-2, xmax=2, ymin=-2, ymax=2,
				no markers,
				samples=50,
				axis lines*=left, 
				axis lines*=middle, 
				scale only axis,
				xtick={-1,  0,  1},
				ytick={0, 2},
				] 
				\addplot[line width=2.0pt, darkred, domain=-2:0]{-1};
				\addplot[line width=2.0pt, darkred, domain=0:2]{1};
				\addplot[line width=2.0pt, smooth, domain=0:2,darkred] coordinates {(0,-1)(0,1)};
			\end{axis}
		\end{tikzpicture}
		\caption{$\partial |\cdot|(0) = [-1,1].$ }
	\end{subfigure}
	~
	\begin{subfigure}[t]{0.3\textwidth}
		\begin{tikzpicture}
			\begin{axis}[
				thick,
				axis lines = middle,
				enlargelimits = true,
				width=.99\textwidth, height=3cm,
				xmin=-2, xmax=2, ymin=-2, ymax=2,
				no markers,
				samples=50,
				axis lines*=left, 
				axis lines*=middle, 
				scale only axis,
				xtick={-1,  0,  1},
				ytick={0, 2},
				] 
				\addplot[line width=2.0pt, darkred, domain=-2:0]{1};
				\addplot[line width=2.0pt, darkred, domain=0:2]{-1};
				\addplot [only marks, mark=o, fill=white] table {
					0  -1
					0  1
				};
			\end{axis}
		\end{tikzpicture}
		\caption{$\partial -|\cdot|(0) = \{-1,1\}.$ }
	\end{subfigure}
	
	\caption{Subgradients are illustrated for the following three cases: (a) smooth function $f(x)=x^2$, (b) a nonsmooth function $f(x)=|x|$, (c) a nonsmooth and nonconvex function $f(x)=|x|$. Subplots (d) to (f) show the corresponding subgradients. }
	\label{fig:subgradients}
\end{figure}
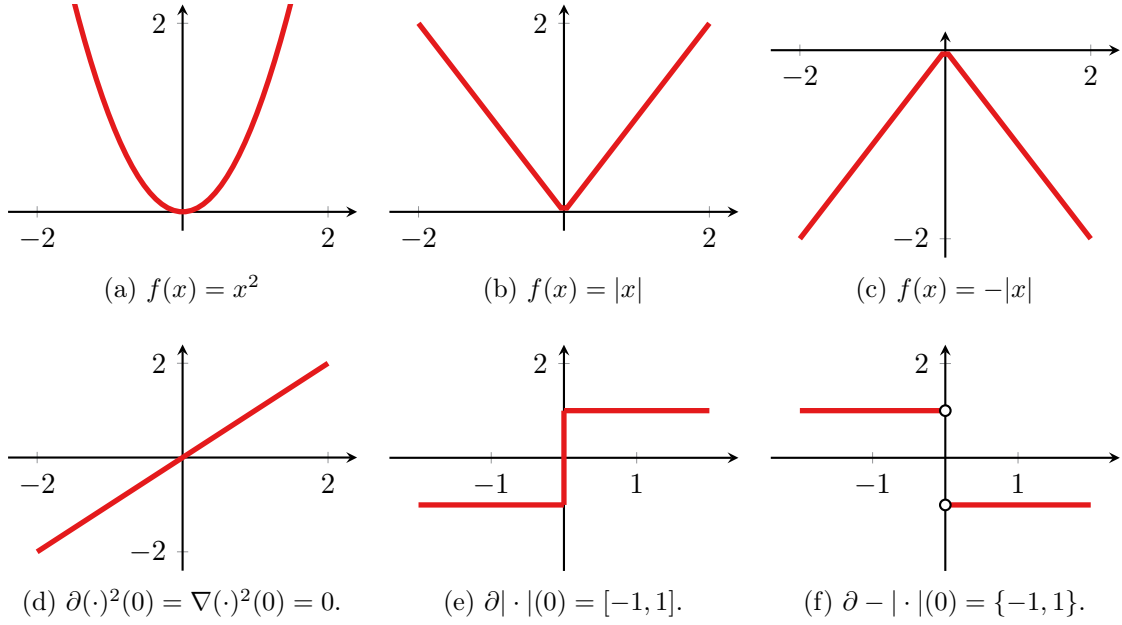

\subsubsection{Moreau Envelope and Proximal Mapping} 
\label{sec:Moreau}
For any function $f$ and real $\gamma>0$, the {\em Moreau envelope} and the 
{\em proximal mapping}  are defined by 
\begin{equation}\label{eq:prox}
	\begin{aligned}
		f_{\gamma}(\mbf M)&:=\inf_{\mbf L}\, \left\{ f(\mbf L)+\frac{1}{2\gamma}\|\mbf L-\mbf M\|^2\right\}, \\
		\prox_{{\gamma}f}(\mbf M) &:=\argmin_{\mbf L}\, \left\{ f(\mbf L)+\frac{1}{2{\gamma}}\|\mbf L-\mbf M\|^2\right\}.
	\end{aligned}
\end{equation}
\begin{theorem}[Regularization properties of the envelope]\label{th:lip_cont} Let $f\colon\R^d\to\R$ be a proper closed convex function. Then $f_{\gamma}$ is convex and $C^1$-smooth with
	$$\nabla f_{\gamma}(\mbf M)=\tfrac{1}{\gamma}(\mbf M-\prox_{\gamma f}(\mbf M))
	\quad \textrm{ and }\quad \Lip(\nabla f_{\gamma})\leq \tfrac{1}{\gamma}.$$
\end{theorem}
\begin{proof}
	See Theorem 2.26 of \cite{RW98}.
\end{proof}

\subsection{Optimality Condition}

In the main text, we show some examples of the prox operator corresponding to $\partial \psi(\mbf B)$, but for $\partial \varphi(\mbf A)$ it might be hard 
to understand.
Here we give a simple instance of $\partial \varphi(\mbf A)$ when $\phi(\mbf A) = \delta_0(\mbf A | \mbf A^\top \mbf A = \mbf I) = \delta_0(\mbf A | \mathbb{O}_2)$, the space of orthonormal matrices. 
When consider orthogonal matrices in two dimension, we could characterize them by a single angle variable $\theta$,
\[
\mbf A = \begin{bmatrix}
	\cos(\theta) & \sin(\theta)\\
	-\sin(\theta) & \cos(\theta)
\end{bmatrix}, \, \text{if } \det(\mbf A) = 1,\quad
\mbf A = \begin{bmatrix}
	\cos(\theta) & \sin(\theta)\\
	\sin(\theta) & -\cos(\theta)
\end{bmatrix}, \, \text{if } \det(\mbf A) = -1.
\]
For every $\mbf A$ described as above, we define the tangent direction in $\mathbb{O}_2$:
\[
\mbf T_{\mbf A} = \begin{bmatrix}
	-\sin(\theta) & \cos(\theta)\\
	-\cos(\theta) & -\sin(\theta)
\end{bmatrix}, \, \text{if } \det(\mbf A) = 1, \quad
\mbf T_{\mbf A} = \begin{bmatrix}
	-\sin(\theta) & \cos(\theta)\\
	\cos(\theta) & \sin(\theta)
\end{bmatrix}, \, \text{if } \det(\mbf A) = -1.
\]
We now have 
\[
\partial \varphi(\mbf A) = \{\mbf G: \ip{\mbf G}{\mbf T_{\mbf A}} = 0\}.
\]
In particular, for $\det(\mbf A) = 1$ every element in $\partial \varphi(\mbf A)$ is a linear combination of the matrices 
\[
\begin{bmatrix}
	-\cos(\theta)  & 0  \\ \sin(\theta) & 0 
\end{bmatrix}, \quad
\begin{bmatrix}
	\cos(\theta) & \sin(\theta)  \\  0 & 0 
\end{bmatrix}, \quad 
\begin{bmatrix}
	0& 0\\  \sin(\theta) & -\cos(\theta) 
\end{bmatrix}.
\]

\subsection{Proof for Theorem~\ref{them:convergence}}
\label{sec:proof_convergence}

\begin{proof}
	By definition, the iterates of Algorithm~\ref{alg:vp_f} satisfy 
	\begin{align*}
		\tfrac{1}{\gamma}(\mbf B_k - \mbf B_{k+1})
		+\partial \psi(\mbf B_k) - \partial \psi(\mbf B_{k+1})
		\in &\mbf X^\top(\mbf X \mbf B_k - \mbf X \mbf A_k) + \partial \psi(\mbf B_k)\\
		\bm 0 \in &(\mbf A_{k+1}\mbf X^\top\mbf B_{k+1}^\top - \mbf X^\top)\mbf B_{k+1} \mbf X + \partial \varphi(\mbf A_{k+1}).
	\end{align*}
	From the definition of the objective, we have,
	\begin{align*}
		f(\mbf A_{k+1}, \mbf B_{k+1})
		= & \frac{1}{2}\|\mbf X - \mbf X \mbf B_{k+1}\mbf A_{k+1}^\top\|_F^2 + \psi(\mbf B_{k+1}) + \varphi(\mbf A_{k+1})\\
		\le & \frac{1}{2}\|\mbf X - \mbf X \mbf B_{k+1}\mbf A_k^\top\|_F^2 + \varphi(\mbf A_k) + \psi(\mbf B_{k+1})\\
		= & \frac{1}{2}\|\mbf X - \mbf X \mbf B_k\mbf A_k^\top + \mbf X \mbf B_k\mbf A_k^\top - \mbf X \mbf B_{k+1}\mbf A_k^\top\|_F^2
		+ \varphi(\mbf A_k) + \psi(\mbf B_{k+1})\\
		= & \frac{1}{2}\|\mbf X - \mbf X\mbf B_k\mbf A_k^\top\|_F^2 + \ip{\mbf X (\mbf A_k - \mbf B_k)}{\mbf X(\mbf B_k - \mbf B_{k+1})}\\
		& + \frac{1}{2}\|\mbf X(\mbf B_k - \mbf B_{k+1})\|_F^2 + \psi(\mbf B_{k+1})+ \varphi(\mbf A_k)\\
		= & f(\mbf A_k, \mbf B_k) + \ip{\mbf X (\mbf A_k - \mbf B_k)}{\mbf X(\mbf B_k - \mbf B_{k+1})}\\
		& + \frac{1}{2}\|\mbf X(\mbf B_k - \mbf B_{k+1})\|_F^2 + \psi(\mbf B_{k+1}) - \psi(\mbf B_k).
	\end{align*}
	Since $\psi$ is a convex function, we have,
	\[
	\psi(\mbf B_{k+1}) - \psi(\mbf B_k)
	\le \ip{\partial \psi(\mbf B_{k+1})}{\mbf B_{k+1} - \mbf B_k}.
	\]
	Therefore, 
	\begin{align*}
		f(\mbf A_{k+1}, \mbf B_{k+1}) - f(\mbf A_k, \mbf B_k)
		\le & \ip{\mbf X^\top\mbf X(\mbf A_k - \mbf B_k)}{\mbf B_k - \mbf B_{k+1}} +\frac{1}{2}\|\mbf X(\mbf B_k - \mbf B_{k+1})\|_F^2\\
		& + \ip{\frac{1}{\gamma}(\mbf B_k - \mbf B_{k+1}) + \mbf X^\top\mbf X(\mbf A_k - \mbf B_k)}{\mbf B_{k+1} - \mbf B_k}\\
		= & -\frac{1}{\gamma}\|\mbf B_k - \mbf B_{k+1}\|_F^2 + \frac{1}{2}\|\mbf X(\mbf B_k - \mbf B_{k+1})\|_F^2\\
		\le & -\frac{1}{2}\|\mbf X\|_2^2\|\mbf B_k - \mbf B_{k+1}\|_F^2.
	\end{align*}
	Using the definition of optimality condition $T$, we have 
	\[\begin{aligned}
		T(\mbf A_k, \mbf B_k) &\le \left(\frac{1}{\|\mbf X\|_2^2} + L\right)^2\|\mbf B_k - \mbf B_{k+1}\|_F^2\\
		&\le \frac{2(\|\mbf X\|_2^2 + L)^2}{\|\mbf X\|_2^2}(f(\mbf A_k, \mbf B_k) - f(\mbf A_{k+1}, \mbf B_{k+1}))
	\end{aligned}\]
	Adding up the terms across $k$, we have a telescoping series on the right hand side, and immediately 
	obtain the result. 
\end{proof}

\clearpage
{\normalsize
\bibliographystyle{plainnat} 
\bibliography{spca_bib}   
}

\end{document}